\documentclass[	
	paper = a4,
	abstract = true
]{scrartcl}
\usepackage[T1]{fontenc}
\usepackage[
pdfa,
pdftex, 
pdfdisplaydoctitle,
pdfpagelabels,
pdfauthor={Nils M. Kriege, Marion Neumann, Christopher Morris, Kristian Kersting, Petra Mutzel},
pdftitle={A unifying view of explicit and implicit feature maps of graph kernels},
pdfsubject={},
pdfkeywords={graph kernels, feature maps, random walk kernel, structured data, supervised learning},
pdfproducer={Latex with the hyperref package},
pdfcreator={pdflatex}
]{hyperref}
\usepackage[auth-lg]{authblk}

\usepackage{graphicx}
\usepackage{xspace}
\usepackage[numbers]{natbib}
\usepackage{paralist}
\usepackage{enumitem}
\setlist[enumerate]{noitemsep, topsep=0.5\topsep}
\setlist[description]{noitemsep, topsep=0.5\topsep}
\usepackage{amsmath}
\usepackage{amssymb}
\usepackage{amsthm}
\usepackage{amsfonts}
\usepackage{thmtools}
\usepackage{mleftright}
\usepackage{dsfont}

\usepackage{thm-restate}
\usepackage[mathic=true]{mathtools}
\usepackage{fixmath}

\theoremstyle{plain}
\newtheorem{theorem}{Theorem}[section]
\newtheorem{proposition}[theorem]{Proposition}
\newtheorem{lemma}[theorem]{Lemma}
\newtheorem{corollary}[theorem]{Corollary}
\theoremstyle{definition}
\newtheorem{remark}[theorem]{Remark}
\newtheorem{definition}[theorem]{Definition}

\usepackage{url}
\usepackage[hang,tight]{subfigure}
\usepackage{mleftright}
\usepackage[english,linesnumbered,vlined,ruled,nokwfunc]{algorithm2e}
\DontPrintSemicolon
\SetKwComment{note}{$\triangleright$ }{}
\SetFuncSty{textsc}
\SetCommentSty{textit}
\SetKwInOut{Input}{Input}
\SetKwInOut{Data}{Data}
\SetKwInOut{Output}{Output}
\let\oldnl\nl
\newcommand{\nonl}{\renewcommand{\nl}{\let\nl\oldnl}}
\SetKwProg{customProcedure}{Procedure}{}{}
\newcommand{\Procedure}[2]{\BlankLine\nonl\customProcedure{#1}{#2}}

\usepackage{microtype}
\usepackage{ellipsis}
\usepackage{xspace}
\usepackage{lmodern}
\usepackage{microtype}
\usepackage{booktabs}
\usepackage{multirow}
\newcommand{\ra}[1]{\renewcommand{\arraystretch}{#1}}

\newcommand{\tp}[0]{\top}%
\newcommand{\norm}[1]{\left\lVert#1\right\rVert}
\newcommand{\setST}{\ \middle|\ } 
\newcommand{\sizeV}[1]{|V(#1)|}
\newcommand{\sizeE}[1]{|E(#1)|}
\newcommand{\lab}[0]{\tau}
\newcommand{\wdpg}[0]{\times_{\hspace{-.15em}w}}
\newcommand{\bbR}[0]{\ensuremath{\mathbb{R}}\xspace}
\newcommand{\bbRnn}[0]{\ensuremath{\mathbb{R}_{\geq0}}\xspace}
\newcommand{\bbN}[0]{\ensuremath{\mathbb{N}}\xspace}
\newcommand{\bigO}[0]{\ensuremath{\mathcal{O}}}
\newcommand{\EdgeSet}[1]{\ensuremath{[#1]^2}}
\newcommand{\wlength}[0]{\ensuremath{\ell}}
\newcommand{\hilb}[0]{\ensuremath{\mathcal{H}}\xspace}
\newcommand{\cG}{\ensuremath{{\mathcal G}}\xspace}
\newcommand{\X}{\ensuremath{\mathcal{X}}\xspace}
\DeclarePairedDelimiter\multiset{\lbrace\!\!\lbrace}{\rbrace\!\!\rbrace}%
\DeclareMathOperator{\N}{N}
\DeclareMathOperator{\REF}{ref}

\addtokomafont{disposition}{\rmfamily}
\addtokomafont{descriptionlabel}{\rmfamily}

\begin{document}

\title{A Unifying View of Explicit and Implicit Feature Maps of Graph Kernels\thanks{%
A preliminary version of this paper appeared in the proceedings of the IEEE International 
Conference on Data Mining (ICDM) in 2014~\citep{Kri+2014}. \newline
This work has been supported by the Deutsche Forschungsgemeinschaft (DFG) within 
the Collaborative Research Center SFB 876 ``Providing Information by 
Resource-Constrained Data Analysis'', project A6 ``Resource-efficient Graph Mining''.}}

\author[1]{Nils M.~Kriege}
\author[2]{Marion Neumann}
\author[1]{Christopher Morris}
\author[3]{Kristian Kersting}
\author[1]{Petra Mutzel}

\affil[1]{Department of Computer Science \authorcr 
TU Dortmund University, Germany \authorcr
\texttt{\{nils.kriege,christopher.morris,petra.mutzel\}@tu-dortmund.de}}

\affil[2]{Department of Computer Science and Engineering \authorcr
Washington University in St.~Louis, USA \authorcr
\texttt{m.neumann@wustl.edu}}

\affil[3]{Computer Science Department and Centre for Cognitive Science \authorcr
TU Darmstadt, Germany \authorcr
\texttt{kersting@cs.tu-darmstadt.de}}

\date{\vspace{-1cm}}

\maketitle

\begin{abstract}
Non-linear kernel methods can be approximated by fast linear ones using suitable 
explicit feature maps allowing their application to large scale problems. 
We investigate how convolution kernels for structured data are composed from 
base kernels and construct corresponding feature maps.
On this basis we propose exact and approximative feature maps for widely used
graph kernels based on the kernel trick.
We analyze for which kernels and graph properties computation by explicit feature 
maps is feasible and actually more efficient. 
In particular, we derive approximative, explicit feature maps for state-of-the-art 
kernels supporting real-valued attributes including the GraphHopper and graph 
invariant kernels.
In extensive experiments we show that our approaches often achieve a classification 
accuracy close to the exact methods based on the kernel trick, but require
only a fraction of their running time.
Moreover, we propose and analyze algorithms for computing random walk, shortest-path
and subgraph matching kernels by explicit and implicit feature maps.
Our theoretical results are confirmed experimentally by observing a 
phase transition when comparing running time with respect to label diversity, 
walk lengths and subgraph size, respectively.
\end{abstract}

\section{Introduction}

Analyzing complex data is becoming more and more important. In numerous application 
domains, e.g., chem- and bioinformatics, neuroscience, or image and social network analysis,
the data is structured and hence can naturally be represented as graphs. 
To achieve successful learning we need to exploit the rich information inherent 
in the graph structure and the annotations of vertices and edges. 
A popular approach to mining structured data is to design graph kernels 
measuring the similarity between pairs of graphs. 
The graph kernel can then be plugged into a kernel machine, such as support 
vector machine or Gaussian process, for efficient learning and prediction.  

The kernel-based approach to predictive graph mining requires a positive semidefinite 
(p.s.d.) kernel function between graphs. Graphs, composed of labeled vertices and edges, 
possibly enriched with continuous attributes, however,  are not fixed-length vectors but 
rather complicated data structures, and thus standard kernels cannot be used. 
Instead, the general strategy to design graph kernels is to decompose graphs
into small substructures among which kernels are defined following the concept 
of convolution kernels due to \citet{Haussler1999}.
The graph kernel itself is then a combination of the kernels between the possibly 
overlapping parts. 
Hence the various graph kernels proposed in the literature mainly differ in the 
way the parts are constructed and in the similarity measure used to compare them.
Moreover, existing graph kernels differ in their ability to exploit annotations,
which may be categorical labels or real-valued attributes on the vertices and edges.

We recall basic facts on kernels, which have decisive implications on 
computational aspects.
A \emph{kernel} on a non-empty set $\X$ is a positive semidefinite function 
$k \colon \X \times \X \to \bbR$.
Equivalently, a function $k$ is a kernel if there is a \emph{feature map} 
$\phi \colon \X \to \hilb$ to a real Hilbert space \hilb with inner product 
$\langle \cdot, \cdot \rangle$, such that 
$k(x,y) = \langle \phi(x),\phi(y) \rangle$ for all $x$ and $y$ in $\X$.
This equivalence yields two algorithmic strategies to compute kernels on graphs:
\begin{enumerate}[label=(\roman*)]
 \item One way is functional computation, e.g., from closed-form expressions.
   In this case the feature map is not necessarily known and the feature space 
   may be of infinite dimension. Therefore, we refer to this approach based on 
   the famous kernel trick as \emph{implicit} computation.
 \item The other strategy is to compute the feature map $\phi(G)$ for each graph 
   $G$ \emph{explicitly} to obtain the kernel values from the dot product 
   between pairs of feature vectors. These feature vectors commonly count how often 
   certain substructures occur in a graph.
\end{enumerate}
The used strategy has a crucial effect on the running time of the kernel 
method at the higher level.
Kernel methods supporting implicit kernel computation are often slower than 
linear ones based on explicit feature maps assuming that the feature vectors 
are of a manageable size.
The running time for training support vector machines, for example, is linear in 
the training set size when assuming that the feature vectors have a 
constant number of non-zero components~\citep{Joachims2006}.
For this reason, approximative explicit feature maps of various popular kernels 
for vectorial data have been studied extensively~\citep{Rahimi2008,Vedaldi2012}.
This, however, is not the case for graph kernels, which are typically proposed using one 
method of computation, either implicit or explicit. 
Graph kernels using explicit feature maps essentially transform graphs into 
vectorial data in a preprocessing step. These kernels are scalable, but are
often restricted to graphs with discrete labels.
Unique advantages of the implicit computation are that
\begin{inparaenum}[(i)]
 \item kernels for composed objects can be obtained by combining established kernels on 
  their parts exploiting well-known closure properties of kernels;
 \item the number of possible features may be high---in theory infinite---while 
  the function remains polynomial-time computable.
\end{inparaenum}
Previously proposed graph kernels that are computed implicitly typically support 
specifying arbitrary kernels for vertex annotations, but do not scale to large 
graphs and data sets.
Even when approximative explicit feature maps of the kernel on vertex annotations
are known, it is not clear how to obtain (approximative) feature maps for the 
graph kernel.

\paragraph{Our contribution.}
We study under which conditions the computation of an explicit mapping from graphs 
to a finite-dimensional feature spaces is feasible and efficient. To achieve our 
goal, we discuss feature maps corresponding to closure properties of kernels and
general convolution kernels with a focus on the size and sparsity of their feature 
vectors.
Our theoretical analysis identifies a trade-off between running time and 
flexibility.

Building on the systematic construction of feature maps we obtain new algorithms
for explicit graph kernel computation, which allow to incorporate (approximative)
explicit feature maps of kernels on vertex annotations. Thereby known 
approximation results for kernels on continuous data are lifted to kernels for 
graphs with continuous annotations.
More precisely, we introduce the class of weighted vertex kernels and show that it generalizes
state-of-the-art kernels for graphs with continuous attributes, namely the 
GraphHopper kernel~\citep{Feragen2013} and an instance of the graph invariant 
kernels~\citep{Ors+2015}. We derive explicit feature maps with approximation 
guarantees based on approximative feature maps of the base kernels to compare annotations.
Then, we propose and analyze algorithms for computing fixed length walk kernels
by explicit and implicit feature maps. 
We investigate shortest-path kernels~\citep{Borgwardt2005} and subgraph 
matching kernels~\citep{Kriege2012} and put the related work into the context 
of our systematic study.
Given this, we are finally able to experimentally compare the running times of 
both computation strategies systematically with respect to the label diversity, 
data set size, and substructure size, i.e., walk length and subgraph size. 
As it turns out, there exists a computational phase transition for walk and 
subgraph kernels.
Our experimental results for weighted vertex kernels show that their computation 
by explicit feature maps is feasible and provides a viable alternative even when
comparing graphs with continuous attributes.

\paragraph{Extension of the conference paper.}
The present paper is a significant extension of a previously published conference 
paper~\citep{Kri+2014}. In the following we list the main contributions that 
were not included in the conference version.

\begin{itemize}
	\item \textit{Feature maps of composed kernels.}
We review closure properties of kernels, the corresponding feature maps and the 
size and sparsity of the feature vectors.
Based on this, we obtain explicit feature maps for convolution kernels with 
arbitrary base kernels. 
This generalizes the result of the conference paper, where binary base kernel
were considered. 
	\item \textit{Weighted vertex kernels.}
We introduce weighted vertex kernels for attributed graphs, which generalize
the GraphHopper kernel~\citep{Feragen2013} and graph invariant kernels~\citep{Ors+2015}.
Weighted vertex kernels were not considered in the conference paper.
	\item \textit{Construction of explicit feature maps.}
We derive explicit feature maps for weighted vertex kernels and the shortest-path 
kernel~\citep{Borgwardt2005} supporting base kernels with explicit feature maps 
for the comparison of attributes. We prove approximation guarantees in case of
approximative feature maps of base kernels.
This contribution is not contained in the conference paper, where only the explicit 
computation of the shortest-path kernel for graphs with discrete labels was discussed.
	\item \textit{Fixed length walk kernels.}
We generalize the explicit computation scheme to support arbitrary vertex and 
edge kernels with explicit feature maps for the comparison of attributes.
In the conference paper only binary kernels were considered.
Moreover, we have significantly expanded the section on walk kernels by spelling out 
all proofs, adding illustrative figures and clarifying the relation to the $k$-step 
random walk kernel as defined by \citet{Sug+2015}.
	\item \textit{Experimental evaluation.}
We largely extended our evaluation, which now includes experiments for the novel 
computation schemes of graph kernels as well as a comparison between a graphlet 
kernel and the subgraph matching kernel~\citep{Kriege2012}.
\end{itemize}

\paragraph{Outline.}
In Section~\ref{sec:related-work} we discuss the related work and proceed by 
fixing the notation in Section~\ref{sec:preliminary}. 
In Section~\ref{sec:maps} we review closure properties of kernels and the 
corresponding feature maps. Moreover, we derive feature maps for general 
convolution kernels.
In Section~\ref{sec:graph_maps} we propose algorithms for computing graph kernels 
and systematically construct explicit feature maps building on these results.
We introduce weighted vertex kernels and derive (approximative) explicit feature 
maps. 
We discuss the fixed length walk kernel and propose algorithms for explicit and 
implicit computation. 
Moreover, we discuss the shortest-path graph kernel as well as the graphlet and
subgraph matching kernel regarding explicit and implicit computation. 
Section~\ref{sec:evaluation} presents the results of our experimental evaluation.

\section{Related work}\label{sec:related-work}
In the following we review existing kernels based on explicit or implicit 
computation and discuss embedding techniques for attributed graphs.
We focus on the approaches most relevant for our work and refer the reader
to the survey articles~\citep{Vishwanathan2010,Gho+2018,Kriege2019} for a more 
comprehensive overview.

\subsection{Graph kernels}
Most graph kernels decompose graphs into substructures and count their occurrences 
to obtain a feature vector. The kernel function then counts the co-occurrences of 
features in two graphs by taking the dot product between their feature vectors.
The \emph{graphlet kernel}, for example, counts induced subgraphs of size $k \in \{3,4,5\}$ of unlabeled graphs according to
  $K(G,H) = \mathbf{f}_{G}^\top \mathbf{f}_{H}$,
where $\mathbf{f}_{G}$ and $\mathbf{f}_{H}$ are the subgraph feature vectors of $G$ and $H$,
respectively~\citep{Shervashidze2009}. 
The \emph{cyclic pattern kernel} is based on cycles and trees and maps the graphs 
to substructure indicator features, which are independent of the substructure
frequency~\citep{Horv'ath2004}. The \emph{Weisfeiler-Lehman subtree kernel} counts 
label-based subtree patterns according to
 $K_d(G, H) = \sum_{i = 1}^h K(G_i, H_i)$, 
where $K(G_i, H_i) =  \langle \mathbf{f}^{(i)}_{G}\mathbf{f}^{(i)}_{H} \rangle$ 
and $\mathbf{f}^{(i)}_{G}$ is a feature vector counting subtree-patterns in $G$ of depth 
$i$~\citep{Shervashidze2009a,Shervashidze2011}. 
A subtree-pattern is a tree rooted at a particular vertex where each level contains the 
neighbors of its parent vertex; the same vertices can appear repeatedly. 
Other graph kernels on subtree-patterns have been proposed in the literature, e.g.,~\citep{Ramon2003,Harchaoui2007,Bai+2015,Hido+2009}. 
In a similar spirit, the \emph{propagation kernel}  iteratively counts similar label or 
attribute distributions to create an explicit feature map for efficient kernel
computation~\citep{Neu+2016}. 
\citet{Martino2012} proposed to decompose graphs into multisets of ordered 
directed acyclic graphs, which are compared by extended tree kernels.
While convolution kernels decompose graphs into their parts and sum over all 
pairs, assignment kernels are obtained from an optimal bijection between 
parts~\citep{Frohlich2005}. Since this does not lead to valid kernels in 
general~\citep{Vert2008,Vishwanathan2010}, various approaches to overcome this 
obstacle have been developed~\citep{Johansson2015,Schiavinato2015,Kriege2016b,Nikolentzos2017}.
Several kernels have been proposed with the goal to take graph structure at 
different scales into account, e.g., using $k$-core 
decomposition~\citep{Nikolentzos2018} or spectral properties~\citep{Kondor2016}.
\citet{Yan+2015} combine neural techniques from language modeling with state-of-the-art 
graph kernels in order to incorporate similarities between the individual substructures.
Such similarities were specifically designed for the substrutures used by the 
graphlet and the Weisfeiler-Lehman subtree kernel, among others.
\citet{Narayanan2016} discuss several problems of the proposed approach to obtain 
substructure similarities and introduce \emph{subgraph2vec} to overcome these issues.

Many real-world graphs have continuous attributes such as real-valued vectors attached 
to their vertices and edges. For example, the vertices of a molecular graph may be annotated 
by the physical and chemical properties of the atoms they represent. 
The kernels based on counting co-occurrences described above, however, consider two substructures 
as identical if they match exactly, structure-wise as well as attribute-wise, 
and as completely different otherwise.
For attributed graphs it is desirable to compare annotations by more complex 
similarity measures such as the Gaussian RBF kernel.
The kernels discussed in the following allow user-defined kernels for the comparison of 
vertex and edge attributes. Moreover, they compare graphs in a way that takes 
the interplay between structure and attributes into account and are therefore 
suitable for graphs with continuous attributes. 

\emph{Random walk kernels} add up a score for all pairs of walks that two graphs
have in common, whereas vertex and edge attributes encountered on walks can be compared 
by user-specified kernels.
For random walk kernels implicit computation schemes based on product graphs have been 
proposed. The product graph $G_{\times}$ has a vertex for each pair of vertices in the 
original graphs. Two vertices in the product graph are neighbors if the corresponding 
vertices in the original graphs are both neighbors as well. Product graphs have some 
nice properties making them suitable for the computation of graph kernels. First, the 
adjacency matrix $A_{\times}$ of a product graph is the Kronecker product of the 
adjacency matrices $A$ and $A'$ of the original graphs, i.e.,
$A_{\times} = A \otimes A'$, same holds for the weight matrix $W_{\times}$ 
when employing an edge kernel.
Further, there is a one-to-one correspondence between walks on the product graph 
and simultaneous walks on the original graphs~\citep{Gaertner2003}. 
The \emph{random walk kernel} introduced by~\citet{Vishwanathan2010} is 
now given by
\begin{equation}\label{eq:rw}
  K(G,H) = \sum_{l=0}^{\infty} \mu_l q^\tp_\times W^l_\times p_\times,
\end{equation}
where $p_\times$ and $q_\times$ are starting and stopping probability 
distributions and $\mu_l$ coefficients such that the sum converges.
Several variations of the random walk kernel have been introduced 
in the literature. 
The \emph{geometric random walk kernel} originally introduced by \citet{Gaertner2003} 
counts walks with the same sequence of discrete labels and is a predecessor of
the general formulation presented above.
The description of the random walk kernel by \citet{Kashima2003} is motivated by 
a probabilistic view on kernels and based on the idea of so-called \emph{marginalized 
kernels}. The method was extended to avoid tottering and the efficiency was improved
by label refinement~\citep{Mah'e2004}.
Several methods for computing Eq.~\eqref{eq:rw} were proposed by 
\citet{Vishwanathan2010} achieving different running times
depending on a parameter $k$, which is the number of fixed-point iterations, power 
iterations and the effective rank of $W_\times$, respectively.
The running times to compare graphs with $n$ vertices also depend on the edge labels
of the input graphs and the desired edge kernel.
For unlabeled graphs the running time $\bigO(n^3)$ is achieved and $\bigO(dkn^3)$
for labeled graphs, where $d$ is the size of the label alphabet. The same running time 
is obtained for edge kernels with a $d$-dimensional feature space, while $\bigO(kn^4)$ 
time is required in the infinite case. For sparse graphs $\bigO(kn^2)$ is obtained in 
all cases.
Further improvements of the running time were subsequently obtained by non-exact
algorithms based on low rank approximations~\citep{Kang2012}.
These random walk kernels take all walks without a bound on length into account.
However, in several applications it has been reported that only walks up to a 
certain length have been considered, e.g., for the prediction of protein functions 
\citep{Borgwardt2005a} or image classification \citep{Harchaoui2007}.
This might suggest that it is not necessary or even not beneficial to consider the 
infinite number of possible walks to obtain a satisfying prediction accuracy.
Moreover, the phenomenon of \emph{halting} in random walk kernels has been
studied recently~\citep{Sug+2015}, which refers to the fact that long walks are 
down-weighted such that the kernel is in fact dominated by walks of length 1.

Another substructure used to measure the similarity among graphs are shortest paths.
\citet{Borgwardt2005} proposed the \emph{shortest-path kernel}, which compares
two graphs based on vertex pairs with similar shortest-path lengths.
The \emph{GraphHopper kernel} compares the vertices encountered while hopping 
along shortest paths by a user-specified kernel~\citep{Feragen2013}. 
Similar to the graphlet kernel, the \emph{subgraph matching kernel} compares 
subgraphs of small size, but allows to score mappings between them according to 
vertex and edge kernels~\citep{Kriege2012}. 
Further kernels designed specifically for graphs with continuous attributes 
exist~\citep{Ors+2015,Su+2016,Martino2018}.

\subsection{Embedding techniques for attributed graphs}
Kernels for attributed graphs often allow to specify arbitrary kernels for
comparing attributes and are computed using the kernel trick without generating 
feature vectors.
Moreover, several approaches for computing vector representations for attributed 
graphs have been proposed. These, however, do not allow specifying a function for comparing attributes. 
The similarity measure that is implicitly used to compare attributes is typically not known.
This is the case for recent deep learning approaches as well as for some kernels
proposed for attributed graphs.

\paragraph{Deep learning on graphs.}
Recently, a number of approaches to graph classification based upon neural networks have been proposed. Here a vectorial representation for each vertex is learned iteratively from the vertex annotations of its neighbors using a parameterized (differentiable) neighborhood aggregation function. Eventually, the vector representations for the individual vertices are combined to obtain a vector representation for the graph, e.g., by summation.

The parameters of the aggregation function are learned together with the parameters of the classification or regression algorithm, e.g., a neural network. 
More refined approaches use differential pooling operators based on sorting~\citep{Zhang2018} and soft assignments~\citep{Yin+2018}.
Most of these neural approaches fit into the framework proposed by~\citet{Gil+2017}. Notable instances of this model include \emph{neural fingerprints}~\citep{Duv+2015}, \emph{GraphSAGE}~\citep{Ham+2017}, and the spectral approaches proposed by~\citet{Bru+2014}, \citet{Def+2015} and \citet{Kip+2017}---all of which descend from early work, see, e.g.,~\citep{Mer+2005} and~\citep{Sca+2009}.

These methods show promising results on several graph classification benchmarks, see, e.g.,~\citep{Yin+2018}, as well as in applications such as protein-protein interaction prediction~\citep{Fou+2017}, recommender systems~\citep{Yin+2018a}, and the analysis of quantum interactions in molecules~\citep{Sch+2017}.
A survey of recent advancements can be found in \citep{Ham+2017a}.
With these approaches, the vertex attributes are aggregated for each graph 
and not directly compared between the graphs. Therefore, it is not obvious how the 
similarity of vertex attributes is measured.

\paragraph{Explicit feature maps of kernels for attributed graphs.}
Graph kernels supporting complex annotations typically use implicit computation 
schemes and do not scale well. Whereas graphs with discrete labels are efficiently
compared by graph kernels based on explicit feature maps. 
Kernels limited to graphs with categorical labels can be applied to attributed graphs
by discretization of the continuous attributes, see, e.g.,~\citep{Neu+2016}.
\citet{Morris2016} proposed the \emph{hash graph kernel framework} to obtain efficient 
kernels for graphs with continuous labels from those proposed for discrete ones.
The idea is to iteratively turn continuous attributes into discrete
labels using randomized hash functions. A drawback of the approach is that
so-called \emph{independent $k$-hash families} must be known to guarantee that 
the approach approximates attribute comparisons by the kernel $k$.
In practice locality-sensitive hashing is used, which does not provide this 
guarantee, but still achieves promising results. 
To the best of our knowledge no results on explicit feature maps of kernels for  
graphs with continuous attributes that are compared by a well-defined similarity
measure such as the Gaussian RBF kernel are known.

However, explicit feature maps of kernels for vectorial data have been studied extensively.
Starting with the seminal work by~\citet{Rahimi2008}, explicit feature maps of
various popular kernels have been proposed, cf.~\cite[][and references therein]{Vedaldi2012,Kar2012,Pham2013}.
In this paper, we build on this line of work to obtain kernels for graphs, where 
individual vertices and edges are annotated by vectorial data. In contrast to
the hash graph kernel framework our goal is to lift the known approximation results 
for kernels on vectorial data to kernels for graphs annotated with vectorial data.

\section{Preliminaries}\label{sec:preliminary}
An \emph{(undirected) graph} $G$ is a pair $(V,E)$ with a finite set of \emph{vertices} $V$ and a set of \emph{edges} $E \subseteq \{ \{u,v\} \subseteq V \mid u \neq v \}$. We denote the set of vertices and the set of edges of $G$ by $V(G)$ and $E(G)$, respectively. For ease of notation we denote the edge $\{u,v\}$ in $E(G)$ by $uv$ or $vu$ and the set of all graphs by \cG. 
A graph $G' = (V',E')$ is a \emph{subgraph} of a graph $G=(V,E)$ if 
$V' \subseteq V$ and $E' \subseteq E$. 
The subgraph $G'$ is said to be \emph{induced} if $E' = \{uv \in E \mid u,v \in V' \}$ 
and we write $G' \subseteq G$.
We denote the \emph{neighborhood} of a vertex $v$ in $V(G)$ by 
$\N(v) = \{ u \in V(G) \mid vu \in E(G) \}$. 

A \emph{labeled graph} is a graph $G$ endowed with an \emph{label function} 
$\lab \colon V(G) \to \Sigma$, where $\Sigma$ is a finite alphabet. 
We say that $\lab(v)$ is the \emph{label} of $v$ for $v$ in $V(G)$.
An \emph{attributed graph} is a graph $G$ endowed with a function 
$\lab \colon V(G) \to \bbR^d$, $d \in \bbN^+$, and we say that 
$\lab(v)$ is the \emph{attribute} of $v$. We denote the base kernel for comparing
vertex labels and attributes by $k_V$ and, for short, write $k_V(u,v)$ instead of
$k_V(\lab(u),\lab(v))$. The above definitions directly extend to graphs, where
edges have labels or attributes and we denote the base kernel by $k_E$.
We refer to $k_V$ and $k_E$ as \emph{vertex kernel} and \emph{edge kernel},
respectively, and assume both to take non-negative values only.

Let $\mathsf{T}_k$ be the running time for evaluating a kernel for a pair of 
graphs, $\mathsf{T}_\phi$ for computing a feature vector for a single graph and 
$\mathsf{T}_{\text{dot}}$ for computing the dot product between two feature vectors.
Computing an $n \times n$ matrix with all pairwise kernel values for $n$ graphs
requires 
\begin{inparaenum}[(i)]
  \item time $\bigO(n^2 \mathsf{T}_k)$ using implicit feature maps, and
  \item time $\bigO(n \mathsf{T}_{\phi} + n^2 \mathsf{T}_{\text{dot}})$ using
        explicit feature maps.
\end{inparaenum}
Clearly, explicit computation can only be competitive with implicit computation, 
when the time $\mathsf{T}_{\text{dot}}$ is smaller than $\mathsf{T}_k$.
In this case, however, even a time-consuming feature mapping $\mathsf{T}_\phi$ pays 
off with increasing data set size.
The running time $\mathsf{T}_{\text{dot}}$ depends on the data structure used to 
store feature vectors. 
Since feature vectors for graph kernels often contain many components that are 
zero, we consider sparse data structures, which expose running times depending 
on the number of non-zero components instead of the actual number of all components.
For a vector $v$ in $\bbR^d$, we denote by $\mathsf{nz}(v)$ the set of indices 
of the non-zero components of $v$ and let $\mathsf{nnz}(v) = |\mathsf{nz}(v)|$ 
the number of non-zero components.
Using hash tables the dot product between $\Phi_1$ and $\Phi_2$ can be realized 
in time $\mathsf{T}_{\text{dot}}=\bigO(\min\{\mathsf{nnz}(\Phi_1), \mathsf{nnz}(\Phi_2)\})$
in the average case.

\section{Basic kernels, composed kernels and their feature maps}\label{sec:maps}
Graph kernels, in particular those supporting user-specified kernels for annotations,
typically employ closure properties. This allows to decompose graphs into parts 
that are eventually the annotated vertices and edges. The graph kernel then is 
composed of base kernels applied to the annotations and annotated substructures, 
respectively.
We first consider the explicit feature maps of basic kernels and then review 
closure properties of kernels and discuss how to obtain their explicit feature 
maps.
The results are summarized in Table~\ref{tab:closure-map}.
This forms the basis for the systematic construction of explicit feature maps 
of graph kernels according to their composition of base kernels later in 
Section~\ref{sec:graph_maps}.

Some of the basic results on the construction of feature maps and their detailed 
proofs can be found in the text book by \citet{Shawe-Taylor2004}.
Going beyond that, we discuss the sparsity of the obtained feature vectors in
detail.
This has an essential impact on the efficiency in practice, when sparse data 
structures are used and a large number of the components of a feature vector is 
zero.
Indeed the running times we observed experimentally in Section~\ref{sec:evaluation} 
can only be explained taking the sparsity into account.

\subsection{Dirac and binary kernels}\label{sec:maps:dirac}
We discuss feature maps for basic kernels often used for the construction of 
kernels on structured objects.
The Dirac kernel $k_\delta$ on $\X$ is defined by $k_\delta(x,y) = 1$, 
if $x=y$ and $0$ otherwise. For \X a finite set, it is well-known that $\phi \colon \X \to \{0,1\}^{|\X|}$
with components indexed by $i \in \X$ and defined as $\phi(x)_i = 1$ if $i=x$, 
and $0$ otherwise, is a feature map of the Dirac kernel.

The requirement that two objects are equal is often too strict.
When considering two subgraphs, for example, the kernel should take the value $1$
if the graphs are isomorphic and $0$ otherwise. 
Likewise, two vertex sequences corresponding to walks in graphs should be 
regarded as identical if their vertices have the same labels.
We discuss this more general concept of kernels and their properties in the 
following.
We say a kernel $k$ on $\X$ is \emph{binary} if $k(x,y)$ is either $0$ or $1$ 
for all $x,y \in \X$.
Given a binary kernel, we refer to
\begin{equation*}
  \sim_k = \left\{ (x,y) \in \X \times \X \setST k(x,y) = 1 \right\} 
\end{equation*}
as the relation on $\X$ \emph{induced by $k$}. Next we will establish
several properties of this relation, which will turn out to be useful for the
construction of a feature map.

\begin{lemma}\label{lm:bin-sym-psd-zero}
 Let $k$ be a binary kernel on $\X$, then
 $x \sim_k y \Longrightarrow x \sim_k x$ holds for all $x,y \in \X$. 
\end{lemma}
\begin{proof}
 Assume there are $x,y \in \X$ such that $x \not\sim_k x$ and 
 $x \sim_k y$. By the definition of $\sim_k$ we obtain $k(x,x)=0$ and $k(x,y)=1$.
 The symmetric kernel matrix obtained by $k$ for $X=\{x,y\}$ thus is either
$\bigl(\begin{smallmatrix}
 0&1\\
 1&0
\end{smallmatrix} \bigr)$ or $\bigl(\begin{smallmatrix}
 0&1\\
 1&1
\end{smallmatrix} \bigr)$, where we assume that the first row and column is 
associated with $x$. Both matrices are not p.s.d.\ and, thus, $k$ is not a kernel
contradicting the assumption.
\end{proof}

\begin{lemma}\label{lm:bin-sym-psd-relation}
 Let $k$ be a binary kernel on $\X$, then $\sim_k$ is a partial 
 equivalence relation meaning that the relation $\sim_k$ is
 \begin{inparaenum}[(i)]
 \item symmetric, and \label{lm:bin-sym-psd-relation:sym}
 \item transitive. \label{lm:bin-sym-psd-relation:trans}
 \end{inparaenum}
\end{lemma}
\begin{proof}
 Property~\eqref{lm:bin-sym-psd-relation:sym} follows from the fact that $k$ 
 must be symmetric according to definition. Assume property 
 \eqref{lm:bin-sym-psd-relation:trans} does not hold. Then there are 
 $x,y,z \in \X$ with $x \sim_k y \wedge y \sim_k z$ and  $x \not\sim_k z$.
 Since $x \neq z$ must hold according to Lemma~\ref{lm:bin-sym-psd-zero} we can 
 conclude that $X=\{x,y,z\}$ are pairwise distinct.
 We consider a kernel matrix $\vec{K}$ obtained by $k$ for $X$ and assume that 
 the first, second and third row as well as column is associated with $x$, $y$ and
 $z$, respectively.
 There must be entries $k_{12}=k_{21}=k_{23}=k_{32}=1$ and $k_{13}=k_{31}=0$. 
 According to Lemma~\ref{lm:bin-sym-psd-zero} the entries of the main diagonal $k_{11}=k_{22}=k_{33}=1$ follow.
 Consider the coefficient vector $\vec{c}$ with $c_1=c_3=1$ and $c_2=-1$, we 
 obtain $\vec{c}^\tp \vec{K} \vec{c} = -1.$
 Hence, $\vec{K}$ is not p.s.d.\ and $k$ is not a kernel contradicting the assumption.
\end{proof}

We use these results to construct a feature map for a binary kernel. 
We restrict our consideration to the set $\X_{\REF} = \{ x \in \X \mid x \sim_k x \}$, on 
which $\sim_k$ is an equivalence relation. The quotient set $\mathcal{Q}_k=\X_{\REF}/\!\!\sim_k$
is the set of equivalence classes induced by $\sim_k$.
Let $[x]_k$ denote the equivalence class of $x \in \X_{\REF}$ under the relation 
$\sim_k$.
Let $k_\delta$ be the Dirac kernel on the equivalence classes $\mathcal{Q}_k$, 
then $k(x,y) = k_\delta([x]_k,[y]_k)$ and we obtain the following result.
\begin{proposition}\label{prop:bin-equiv-map}
 Let $k$ be a binary kernel with $\mathcal{Q}_k = \{Q_1,\dots,Q_d\}$, then 
 $\phi \colon \X \to \{0,1\}^{d}$ with $\phi(x)_i = 1$ if 
 $Q_i=[x]_k$, and $0$ otherwise, is a feature map of $k$.
\end{proposition}

\begin{table*}
\setlength{\tabcolsep}{4.6pt}
\begin{center}
  \caption{Composed kernels, their feature map, dimension and sparsity. 
  We assume $k=k_1,\dots, k_D$ to be kernels with feature maps $\phi=\phi_1,\dots,\phi_D$ of dimension $d=d_1,\dots,d_D$.}
  \label{tab:closure-map}
  \begin{tabular}{@{}llll@{}}\toprule
    \textbf{Kernel}  & \textbf{Feature Map} & \textbf{Dimension} & \textbf{Sparsity} \\\midrule\vspace{.5em}
    $k^{\alpha}(x,y) = \alpha k(x,y)$ & $\phi^{\alpha}(x) = \sqrt{\alpha}\phi(x)$ & $d$ & $\mathsf{nnz}(\phi(x))$ \\\vspace{.5em}
    $k^{+}(x,y) = \sum_{i =1}^{D}k_i(x,y)$ & $\phi^{+}(x) =  \bigoplus_{i=1}^D \phi_i(x)$ & $\sum^D_{i=1} d_i$ & $\sum_{i=1}^{D} \mathsf{nnz}(\phi_i(x))$ \\\vspace{.5em}
    $k^{\bullet}(x,y) = \prod_{i =1}^{D}k_i(x,y)$ & $\phi^{\bullet}(x) = \bigotimes_{i=1}^D \phi_i(x)$ & $\prod^D_{i=1} d_i$ & $\prod_{i=1}^{D} \mathsf{nnz}(\phi_i(x))$ \\\vspace{.5em}
    $k^\times(X,Y) = \sum_{x \in X} \sum_{y \in Y} k(x,y)$ & $\phi^{\times}(X) = \sum_{x \in X} \phi(x)$ & $d$ & $\left|\bigcup_{x \in X} \mathsf{nz}(\phi(x))\right|$\\
    \bottomrule
  \end{tabular}
\end{center}
\end{table*}

\subsection{Closure properties}\label{sec:maps:closure}

For a kernel $k$ on a non-empty set $\X$ the function $k^{\alpha}(x,y) = \alpha k(x,y)$ with 
$\alpha$ in $\bbR_{\geq 0}$ is again a kernel on $\X$. 
Let $\phi$ be a feature map of $k$, then $\phi^{\alpha}(x) = \sqrt{\alpha}\phi(x)$ 
is a feature map of $k^{\alpha}$. For addition and multiplication, we get the following result.

\begin{proposition}[{\citealt[pp.\,75\,sqq.]{Shawe-Taylor2004}}]\label{prop:add-mult-maps}
Let $k_1, \dots, k_D$ for $D>0$ be kernels on $\X$ with feature maps 
$\phi_1, \dots, \phi_D$ of dimension $d_1, \dots, d_D$, respectively. Then 
\begin{eqnarray*}
 k^{+}(x,y) = \sum_{i =1}^{D}k_i(x,y) 
 \quad\text{and}\quad 
 k^{\bullet}(x,y) = \prod_{i =1}^{D}k_i(x,y)
\end{eqnarray*}
are again kernels on $\X$. Moreover,
\begin{eqnarray*}
 \phi^{+}(x) =  \bigoplus_{i=1}^D \phi_i(x) 
 \quad\text{and}\quad 
 \phi^{\bullet}(x) = \bigotimes_{i=1}^D \phi_i(x)
\end{eqnarray*} 
are feature maps for $k^{+}$ and $k^{\bullet}$ of dimension  
$\sum^D_{i=1} d_i$ and $\prod^D_{i=1} d_i$, respectively. Here
$\oplus$ denotes the concatenation of vectors and $\otimes$ the Kronecker product.
\end{proposition}

\begin{remark}
In case of $k_1 = k_2 = \dots = k_D$, we have $k^{+}(x,y)=D k_1(x,y)$ and a $d_1$-dimensional feature map can be 
obtained. For $k^{\bullet}$ we have $k_1(x,y)^D$, which yet does not allow for a feature space of
dimension smaller than $d_1^D$ in general.
\end{remark}

We state an immediate consequence of Proposition~\ref{prop:add-mult-maps} regarding
the sparsity of the obtained feature vectors explicitly.

\begin{corollary}
 Let $k_1, \dots, k_D$ and  $\phi_1, \dots, \phi_D$ be defined as above, then
\begin{eqnarray*}
 \mathsf{nnz}(\phi^{+}(x)) = \sum_{i=1}^{D} \mathsf{nnz}(\phi_i(x))
 \quad\text{and}\quad 
 \mathsf{nnz}(\phi^{\bullet}(x)) = \prod_{i=1}^{D} \mathsf{nnz}(\phi_i(x)).
\end{eqnarray*}
\end{corollary}

\subsection{Kernels on sets}\label{sec:maps:sets}
In the following we derive an explicit mapping for kernels on finite sets. This 
result will be needed in the succeeding section for constructing an explicit 
feature map for the $R$-convolution kernel. Let $\kappa$ be a base kernel on a set $U$, 
and let $X$ and $Y$ be finite subsets of $U$. Then 
the \emph{cross product kernel} or \emph{derived subset kernel} on $\mathcal{P}(U)$ is defined as
\begin{equation}\label{eq:kernel:crossproduct}
k^\times(X,Y) = \sum_{x \in Y} \sum_{y \in Y} \kappa(x,y)\,.
\end{equation}
Let $\phi$ be a feature map of $\kappa$, then the function
\begin{equation}\label{eq:maps:crossproduct}
\phi^{\times}(X) = \sum_{x \in X} \phi(x)
\end{equation}
is a feature map of the cross product kernel~\citep[Proposition 9.42]{Shawe-Taylor2004}.
In particular, the feature space of the cross product kernel corresponds to the 
feature space of the base kernel; both have the same dimension.
For $\kappa=k_\delta$ the Dirac kernel $\phi^\times(X)$ maps the set $X$ to its 
characteristic vector, which has $|U|$ components and $|X|$ non-zero elements.
When $\kappa$ is a binary kernel as discussed in Section~\ref{sec:maps} the number of 
components reduces to the number of equivalence classes of $\sim_\kappa$ and
the number of non-zero elements becomes the number of cells in the quotient set
$X/\!\!\sim_\kappa$. In general, we obtain the following result as an immediate 
consequence of Equation~\eqref{eq:maps:crossproduct}.
\begin{corollary}
 Let $\phi^{\times}$ be the feature map of the cross product kernel and $\phi$ 
 the feature map of its base kernel, then 
 \begin{equation*}
 \mathsf{nnz}(\phi^{\times}(X)) = \left|\bigcup_{x \in X} \mathsf{nz}(\phi(x))\right|\,.
 \end{equation*}
\end{corollary}
A crucial observation is that the number of non-zero components of a feature 
vector depends on both, the cardinality and structure of the set $X$ and the 
feature map $\phi$ acting on the elements of $X$. It is as large as possible when
each element of $X$ is mapped by $\phi$ to a feature vector with distinct non-zero components.

\subsection{Convolution kernels}\label{sec:maps:r-conv-explicit}
\citet{Haussler1999} proposed $R$-convolution kernels as a generic framework to 
define kernels between composite objects.
In the following we derive feature maps for such kernels by using the basic 
closure properties introduced in the previous sections. Thereby, we 
generalize the result presented in \citep{Kri+2014}.

\begin{definition}\label{def:r-conv}
Suppose 
$x \in \mathcal{R} = \mathcal{R}_1 \times \cdots \times \mathcal{R}_n$ 
are the parts of $X \in \X$ according to some decomposition.
Let $R \subseteq \X \times \mathcal{R}$ be a relation such that 
$(X,x) \in R$ if and only if $X$ can be decomposed into the parts $x$.
Let $R(X) = \{ x \mid (X, x) \in R \}$ and assume $R(X)$ is finite for all 
$X \in \X$.
The \emph{$R$-convolution kernel} is
\begin{equation}\label{eq:kernel:r-conv}
  k^\star(X,Y) = 
  \sum_{x \in R(X)} 
  \sum_{y \in R(Y)}
  \underbrace{\prod_{i=1}^n \kappa_i(x_i,y_i)}_{\kappa(x,y)},
\end{equation}
where $\kappa_i$ is a kernel on $\mathcal{R}_i$ for all $i \in \{1,\dots,n\}$.
\end{definition}

Assume that we have explicit feature maps for the kernels $\kappa_i$.
We first note that a feature map for $\kappa$ can be obtained from the feature 
maps for $\kappa_i$ by Proposition~\ref{prop:add-mult-maps}.%
\footnote{Note that we may consider every kernel $\kappa_i$ on $\mathcal{R}_i$ as 
kernel $\kappa'_i$ on $\mathcal{R}$ by defining
$\kappa'_i(x,y)=\kappa_i(x_i,y_i)$.}
In fact, Equation~\eqref{eq:kernel:r-conv} for arbitrary $n$ can be obtained 
from the case $n=1$ for an appropriate choice of $\mathcal{R}_1$ and $k_1$ as 
noted by \citet{Shin2010}.
If we assume $\mathcal{R} = \mathcal{R}_1 = U$, the $R$-convolution kernel 
boils down to the crossproduct kernel and we have $k^\star(X,Y) = k^\times(R(X),R(Y))$,
where both employ the same base kernel $\kappa$.
We use this approach to develop explicit mapping schemes for graph kernels in 
the following.
Let $\phi$ be a feature map for $\kappa$ of dimension $d$, then
from Equation~\eqref{eq:maps:crossproduct}, we obtain an explicit mapping of 
dimension $d$ for the $R$-convolution kernel according to
\begin{equation}\label{eq:maps:r-conv}
\phi^{\star}(X) = \sum_{x \in R(X)} \phi(x)\,.
\end{equation}
As discussed in Section~\ref{sec:maps:sets} the sparsity of $\phi^{\star}(X)$ 
simultaneous depends on the number of parts and their relation in the feature
space of $\kappa$.

\citet{Kri+2014} considered the special case that $\kappa$ is a binary kernel, 
cf.\@ Section~\ref{sec:maps:dirac}.
From Proposition~\ref{prop:bin-equiv-map} and Equation~\eqref{eq:maps:r-conv} we 
directly obtain their result as special case.

\begin{corollary}[\citet{Kri+2014}, Theorem 3]
 Let $k^\star$ be an $R$-convolution kernel with binary kernel $\kappa$ and
 $\mathcal{Q}_\kappa = \{Q_1,\dots,Q_d\}$, then 
 $\phi^\star \colon \X \to \bbN^d$ with $\phi(x)_i = |Q_i \cap X|$ is a feature map 
 of $k^\star$.
\end{corollary}

\section{Computing graph kernels by explicit and implicit feature maps}\label{sec:graph_maps}
Building on the systematic construction of feature maps of kernels, we discuss 
explicit and implicit computation schemes of graph kernels.
We first introduce weighted vertex kernels. This family of kernels generalizes the 
GraphHopper kernel~\citep{Feragen2013} and graph invariant kernels~\citep{Ors+2015}
for attributed graphs, which were recently proposed with an implicit method of
computation. We derive (approximative) explicit feature maps for weighted vertex 
kernels.
Then, we develop explicit and implicit methods of computation for fixed length
walk kernels, which both exploit sparsity for efficient computation.
Finally, we discuss shortest-path and subgraph kernels for which both computation 
schemes have been considered previously and put them in the context of our 
systematic study.
We empirically study both computation schemes for graph kernels confirming our 
theoretical results experimentally in Section~\ref{sec:evaluation}.

\subsection{Weighted vertex kernels}
Kernels suitable for attributed graphs typically use user-defined kernels for 
the comparison of vertex and edge annotations such as real-valued vectors. The 
graph kernel is then obtained by combining these kernels according to closure 
properties. 
Recently proposed kernels for attributed graphs such as 
GraphHopper~\citep{Feragen2013} and graph invariant kernels~\citep{Ors+2015}
use separate kernel functions for the graph structure and vertex annotations. 
They can be expressed as
\begin{equation}\label{eq:kernel:wv}
K^\text{WV}(G,H) =\! \sum_{v \in V(G)} \sum_{v' \in V(H)} k_W(v,v') \cdot k_V(v,v'),
\end{equation}
where $k_V$ is a user-specified kernel comparing vertex attributes and $k_W$ is 
a kernel that determines a weight for a vertex pair based on the individual graph
structures. 
Hence, in the following we refer to Equation~\eqref{eq:kernel:wv} as 
\emph{weighted vertex kernel}.
Kernels belonging to this family are easily identifiable as instances of 
$R$-convolution kernels, cf.~Definition~\ref{def:r-conv}.

\subsubsection{Weight kernels}\label{sec:wv}
We discuss two kernels for attributed graphs, which have been proposed recently
and can bee seen as instances of weighted vertex kernels.

\paragraph{Graph invariant kernels.}
One approach to obtain weights for pairs of vertices is to compare their 
neighborhood by the classical Weisfeiler-Lehman label refinement~\citep{Shervashidze2011,Ors+2015}.
For a parameter $h$ and a graph $G$ with uniform initial labels $\lab_0$, a sequence 
$(\lab_1,\dots,\lab_h)$ of refined labels referred to as \emph{colors} is computed,
where $\lab_i$ is obtained from $\lab_{i-1}$ by the following procedure.
Sort the multiset of colors 
$\multiset{\lab_{i-1}(u) \mid vu \in E(G)}$ 
for every vertex $v$ to obtain a unique sequence of colors and 
add $\lab_{i-1}(v)$ as first element. Assign a new color $\lab_i(v)$ to every 
vertex $v$ by employing an injective mapping from color sequences to new colors. 
A reasonable implementation of $k_W$ motivated along the lines of 
graph invariant kernels~\citep{Ors+2015} is
\begin{equation}
k_W(v,v') = \sum_{i=0}^h k_\delta(\tau_i(v),\tau_i(v')), 
\end{equation}
where $\tau_i(v)$ denotes the discrete label of the vertex $v$ after the $i$th 
iteration of Weisfeiler-Lehman label refinement of the underlying unlabeled 
graph. Intuitively, this kernel reflects to what extent the two vertices have a 
structurally similar neighborhood.

\paragraph{GraphHopper kernel.}
Another graph kernel, which fits into the framework of weighted vertex kernels, 
is the GraphHopper kernel~\citep{Feragen2013} with
\begin{equation}\label{gh}
k_W(v,v') = \langle \vec{M}(v), \vec{M}(v') \rangle_F\,.
\end{equation}
Here $\vec{M}(v)$ and $\vec{M}(v')$ are $\delta \times \delta$ matrices, where 
the entry $\vec{M}(v)_{ij}$ for $v$ in $V(G)$ counts the number of times the 
vertex $v$ appears as the $i$th vertex on a shortest path of discrete length $j$ 
in $G$, where $\delta$ denotes the maximum diameter over all graphs, and 
$\langle \cdot, \cdot \rangle_F$ is the Frobenius inner product.

\subsubsection{Vertex kernels}
For graphs with multi-dimensional real-valued vertex attributes in $\bbR^d$ one 
could set $k_V$ to the Gaussian RBF kernel $k_{\text{RBF}}$ or the dimension-wise 
product of the hat kernel $k_{\Delta}$, respectively, i.e.,
\begin{equation}\label{eq:rbf_hat}
k_{\text{RBF}}(x,y) = \exp\mleft(-\frac{\norm{x-y}_2^2}{2 \sigma^2}\mright) \quad\text{and}\quad
k_{\Delta}(x,y) = \prod_{i=1}^d \max\left\{0, 1-\frac{|x_i-y_i|}{\delta}\right\}.
\end{equation}
Here, $\sigma$ and $\delta$ are parameters controlling the decrease of the kernel 
value with increasing discrepancy between the two input data points.

\subsubsection{Computing explicit feature maps}
In the following we derive an explicit mapping for weighted vertex kernels. 
Notice that Equation~\eqref{eq:kernel:wv} is an instance of Definition~\ref{def:r-conv}.
Hence, by Proposition~\ref{prop:add-mult-maps} and Equation~\eqref{eq:maps:r-conv},
we obtain an explicit mapping $\phi^\text{WV}$ of weighted vertex kernels.

\begin{proposition}\label{prop:convexp}
 Let $K^\text{WV}$ be a weighted vertex kernel according to Equation~\eqref{eq:kernel:wv} 
 with $\phi^W$ and $\phi^V$ feature maps for $k_W$ and $k_V$, respectively. Then
 \begin{equation}
  \phi^\text{WV}(G) = \sum_{v \in V(G)} \phi^W(v) \otimes \phi^V(v)
 \end{equation}
 is a feature map for $K^\text{WV}$.
\end{proposition}

Widely used kernels for the comparison of attributes, such as the Gaussian RBF kernel, do 
not have feature maps of finite dimension. However, \citet{Rahimi2008} 
obtained finite-dimensional feature maps approximating the kernels $k_{\text{RBF}}$ 
and $k_{\Delta}$ of Equation~\eqref{eq:rbf_hat}.
Similar results are known for other popular kernels for vectorial data like the 
Jaccard~\citep{Vedaldi2012} and the Laplacian kernel~\citep{Andoni2008}.

In the following we approximate $k_V(v,w)$ in Equation~\eqref{eq:kernel:wv} by $\langle \widetilde{\phi}^\text{V}(v), \widetilde{\phi}^\text{V}(w) \rangle$, where $\widetilde{\phi}^\text{V}$ is a finite-dimensional, approximative mapping, such that with probability $(1-\delta)$ for $\delta \in (0,1)$
\begin{equation}\label{eq:approx}
\sup_{v,w \in V(G)} \mleft|\mleft\langle \widetilde{\phi}^\text{V}(v), \widetilde{\phi}^\text{V}(w) \mright\rangle - k_V(v,w) \mright| \leq \varepsilon,
\end{equation}
for any $\varepsilon > 0$,
and derive a finite-dimensional, approximative feature map for weighted vertex kernels.
\begin{proposition}\label{prop:excgk}
Let $K^\text{WV}$ be a weighted vertex kernel and let $\mathcal{D}$ be a non-empty finite set of graphs. Further, let $\phi^W$ be a feature map for $k_W$ and let $\widetilde{\phi}^{\text{V}}$ be an approximative mapping for $k_V$ according to Equation~\eqref{eq:approx}. Then we can compute an approximative feature map $\widetilde{\phi}^\text{WV}$ for $K^\text{WV}$ such that with any constant probability
\begin{equation}
\sup_{G,H \in \mathcal{D}} \mleft|\mleft\langle \widetilde{\phi}^\text{WV}(G), \widetilde{\phi}^\text{WV}(H) \mright\rangle - K^\text{WV}(G,H) \mright| \leq \lambda,
\end{equation}
for any $\lambda > 0$.
\end{proposition}
\begin{proof}
By Inequality~\eqref{eq:approx} we get that for every pair of vertices in the data set $\mathcal{D}$ with any constant probability
\begin{equation*}
\mleft| \mleft\langle \widetilde{\phi}^{\text{V}}(v), \widetilde{\phi}^{\text{V}}(w) \mright\rangle - k_V(v,w) \mright| \leq  \varepsilon\,.
\end{equation*}
By the above, the accumulated error is
 \begin{align*}
&\Big| \sum_{v \in V(G)} \sum_{v' \in V(H)} k_V(v,v') \cdot k_W(v,v') -\!\!\! \sum_{v \in V(G)} \sum_{v' \in V(H)} \mleft\langle \widetilde{\phi}^{\text{V}}(v), \widetilde{\phi}^{\text{V}}(w) \mright\rangle  \cdot k_W(v,v')\Big|\\
&\leq \sum_{v \in V(G)} \sum_{v' \in V(H)} k_W (v,v') \cdot \varepsilon\,.
 \end{align*}
Hence, the result follows by setting $\varepsilon = \lambda/(k^{\text{W}}_{\max} \cdot |V_{\max}|^2)$, where $k^{\text{W}}_{\max}$ is the maximum value attained by the kernel $k^{\text{W}}$ and $|V_{\max}|$ is the maximum number of vertices over the whole data set.
\end{proof}

\subsection{Fixed length walk kernels}\label{sec:flrw}
In contrast to the classical walk based graph kernels, \emph{fixed length walk kernels} 
take only walks up to a certain length into account.
Such kernels have been successfully used in practice~\citep{Borgwardt2005a,Harchaoui2007}
and are not susceptible to the phenomenon of halting~\citep{Sug+2015}.
We propose an explicit and implicit computation scheme for fixed length walk kernels
supporting arbitrary vertex and edge kernels.
Our implicit computation scheme is based on product graphs and benefits from sparse
vertex and edge kernels. 
Previously no algorithms based on explicit mapping for computation of walk-based 
kernels have been proposed.
For graphs with discrete labels, we identify the label diversity and walk lengths as 
key parameters affecting the running time. This is confirmed experimentally in
Section~\ref{sec:evaluation}.

\subsubsection{Basic definitions}
A fixed length walk kernel measures the similarity between graphs based on the 
similarity between all pairs of walks of length $\wlength$ contained in the two 
graphs.
A walk of length $\wlength$ in a graph $G$ is a sequence of vertices and edges 
$(v_0,e_1,v_1,\dots,e_\wlength,v_\wlength)$ such that $e_i= v_{i-1}v_i \in E(G)$ 
for $i \in \{1, \dots, \wlength\}$.
We denote the set of walks of length $\wlength$ in a graph $G$ by $\mathcal{W}_\wlength(G)$.

\begin{definition}[$\wlength$-walk kernel]\label{def:walk-kernel}
  The \emph{$\wlength$-walk kernel} between two attributed graphs $G$ and $H$ in $\cG$ 
  is defined as
  \begin{equation}\label{eq:walk-kernel}
      K^=_\wlength(G, H) =
      \sum_{w \in \mathcal{W}_\wlength(G)} 
      \sum_{w' \in \mathcal{W}_\wlength(H)}
      k_W(w,w'),
  \end{equation}
  where $k_W$ is a kernel between walks.
\end{definition}

Definition~\ref{def:walk-kernel} is very general and does not specify how to 
compare walks. An obvious choice is to decompose walks and define $k_W$ in 
terms of vertex and edge kernel functions, denoted by $k_V$ and 
$k_E$, respectively. We consider 
\begin{equation}\label{eq:two-walk-kernel}
 k_W(w, w') = \prod_{i=0}^\wlength k_V(v_i, v'_i) \prod_{i=1}^\wlength k_E(e_i, e'_i),
\end{equation}
where $w = (v_0,e_1,\dots,v_\wlength)$ and $w' = (v'_0,e'_1,\dots,v'_\wlength)$ are two 
walks.\footnote{%
The same idea to compare walks was proposed by \citet{Kashima2003} as part of 
the marginalized kernel between labeled graphs.}
Assume the graphs in a data set have simple vertex and edge labels 
$\lab \colon V \uplus E \to \mathcal{L}$.
An appropriate choice then is to use the Dirac kernel for both, vertex and edge 
kernels, between the associated labels. In this case two walks are considered
equal if and only if the labels of all corresponding vertices and edges are equal. 
We refer to this kernel by
\begin{equation}\label{eq:two-walk-kernel-dirac}
 k^\delta_W(w, w') = 
  \prod_{i=0}^\wlength k_\delta(\lab(v_i), \lab(v'_i)) \prod_{i=1}^\wlength k_\delta(\lab(e_i), \lab(e'_i)),
\end{equation}
where $k_\delta$ is the Dirac kernel.
For graphs with continuous or multi-dimensional annotations this choice is not 
appropriate and $k_V$ and $k_E$ should be selected depending on the
application-specific vertex and edge attributes.

A variant of the $\wlength$-walk kernel can be obtained by considering all walks
up to length $\wlength$.
\begin{definition}[Max-$\wlength$-walk kernel]
  The \emph{Max-$\wlength$-walk kernel} between two attributed graphs 
  $G$ and $H$ in $\cG$ is defined as
  \begin{equation}
      K^\leq_\wlength (G,H) = \sum_{i=0}^\wlength \lambda_i K^=_i(G, H),
  \end{equation}
  where $\lambda_0, \dots, \lambda_\wlength \in \bbRnn$ are weights.
\end{definition}
This kernel is referred to as \emph{$k$-step random walk kernel} by \citet{Sug+2015}.
In the following we primary focus on the $\wlength$-walk kernel, although our 
algorithms and results can be easily transferred to the Max-$\wlength$-walk kernel.

\subsubsection{Walk and convolution kernels}\label{sec:walk_r-convolution_kernel}
We show that the $\wlength$-walk kernel is p.s.d.\ if $k_W$ is a 
valid kernel by seeing it as an instance of an $R$-convolution kernel. 
We use this fact to develop an algorithm for explicit mapping based on the ideas 
presented in Section~\ref{sec:maps:r-conv-explicit}. 

\begin{proposition}
 The $\wlength$-walk kernel is positive semidefinite if $k_W$ is defined according
 to Equation~\eqref{eq:two-walk-kernel} and $k_V$ and $k_E$ are valid kernels.
\end{proposition}
\begin{proof}
The result follows from the fact that the $\wlength$-walk kernel can be seen 
as an instance of an $R$-convolution kernel, cf.\@  Definition~\ref{def:r-conv},
where graphs are decomposed into walks.
Let $w = (v_0, e_1, v_1, \dots, e_\wlength, v_\wlength) = (x_0, \dots, x_{2\wlength})$
and $w' = (x'_0, \dots, x'_{2\wlength})$
be two walks and $\kappa(w, w')=\prod_{i=0}^{2\wlength} \kappa_i(x_i,x'_i)$ with
\begin{equation*}
    \kappa_i=
        \begin{cases}
        k_V &\text{if $i$ is even,}\\
        k_E &\text{otherwise}
        \end{cases}
\end{equation*}
for $i \in \{0,\dots,2\wlength\}$, then $k_W = \kappa$.
This implies that the $\wlength$-walk kernel is a valid kernel if $k_V$ and 
$k_E$ are valid kernels.
\end{proof}

Since kernels are closed under taking linear combinations with non-negative
coefficients, see Section~\ref{sec:maps}, we obtain the following corollary.
\begin{corollary}
 The Max-$\wlength$-walk kernel is positive semidefinite.
\end{corollary}

\subsubsection{Implicit kernel computation}\label{sec:flrw:kernel-computation:implicit}
An essential part of the implicit computation scheme is the generation of 
the product graph that is then used to compute the $\wlength$-walk kernel.

\paragraph{Computing direct product graphs.}\label{sec:product-graph}
In order to support graphs with arbitrary attributes, vertex and edge kernels 
$k_V$ and $k_E$ are considered as part of the input.
Product graphs can be used to represent these kernel values between pairs of 
vertices and edges of the input graphs in a compact manner. We avoid to create
vertices and edges that would represent incompatible pairs with kernel value 
zero. The following definition can be considered a weighted version of the 
direct product graph introduced by~\citet{Gaertner2003} for kernel 
computation.\footnote{Note that we consider undirected graphs while 
\citet{Gaertner2003} refers to directed graphs.}

\begin{definition}[Weighted Direct Product Graph] \label{def:weighted-direct-productgraph}
 For two attributed graphs $G=(V,E)$, $H=(V',E')$ and given vertex and edge kernels 
 $k_V$ and $k_E$, the \emph{weighted direct product graph} (WDPG) is 
 denoted by $G \wdpg H = (\mathcal{V}, \mathcal{E},w)$ and defined as
 \begin{align*}
  \mathcal{V}\ =\ & \left\{ (v,v') \in V \times V' \setST k_V(v,v') > 0 \right\} \hspace{-23em}&\hspace{6em} \\
  \mathcal{E}\ =\ & \left\{ (u,u')(v,v') \in \EdgeSet{\mathcal{V}} \setST uv \in E 
                 \wedge u'v' \in E' \wedge k_E(uv,u'v') > 0 \right\} \hspace{-23em}&\hspace{6em}\\
  w(v)\ =\ & k_V(u,u') && \forall v=(u,u') \in \mathcal{V} \\
  w(e)\ =\ & k_E(uv,u'v') && \forall e \in\mathcal{E}, \text{where $e=(u,u')(v,v')$.}
 \end{align*}
Here $\EdgeSet{\mathcal{V}}$ denotes the set of all $2$-element subsets of $\mathcal{V}$.
\end{definition}

\begin{figure}
   \centering
   \null\hfill
   \subfigure[$G$] {
      \includegraphics{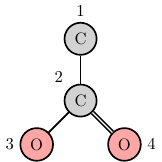}
      \label{fig:directproduct:input1}
   }
   \hfill
   \subfigure[$H$] {
      \includegraphics{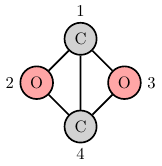}
      \label{fig:directproduct:input2}
   }
   \hfill
   \subfigure[$G \wdpg H$] {
      \includegraphics{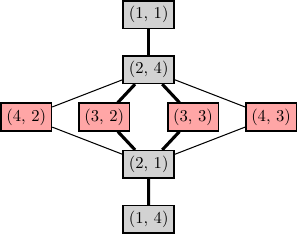}
      \label{fig:directproduct:example}
   }
   \hfill\null
   \caption{Two attributed graphs $G$~\subref{fig:directproduct:input1} and 
     $H$~\subref{fig:directproduct:input2} and their weighted direct product 
     graph $G \wdpg H$~\subref{fig:directproduct:example}. We assume the vertex
     kernel to be the Dirac kernel and $k_E$ to be $1$ if edge labels are 
     equal and $\frac{1}{2}$ if one edge label is ``\texttt{=}'' and the other is ``\texttt{-}''.
     Thin edges in $G \wdpg H$ represent edges with weight $\frac{1}{2}$, while 
     all other edges and vertices have weight $1$.
   }
   \label{fig:directproduct}
\end{figure}

An example with two graphs and their weighted direct product graph obtained for 
specific vertex and edge kernels is shown in Figure~\ref{fig:directproduct}.
Algorithm~\ref{alg:wdpg} computes a weighted direct product graph and does not 
consider edges between pairs of vertices $(v,v')$ that have been identified as
incompatible, i.e., $k_V(v,v')=0$. 

\SetKwFunction{WDPG}{Wdpg}%
\begin{algorithm}
  \caption{Weighted Direct Product Graph}
  \label{alg:wdpg}
  \Input{Graphs $G$ and $H$, vertex and edge kernels $k_V$ and $k_E$.}
  \Output{Graph $G \wdpg H = (\mathcal{V},\mathcal{E},w)$.}

  \Procedure{\WDPG{$G,H,k_V,k_E$}}{
    \ForAll{$v \in V(G)$, $v' \in V(H)$}{
      $w \gets k_V(v,v')$ \;
      \If{$w > 0$}{
        create vertex $z = (v,v')$ \;
        $\mathcal{V} \gets \mathcal{V} \uplus \{z\}$ \;
        $w(z) = w$ \;
      }
    }

    \ForAll{$(u,s) \in \mathcal{V}$}{ \label{alg:wdpg:vp-ckeck1}\label{alg:wdpg:start-edge}
      \ForAll{$v \in \N(u), t \in \N(s)$ with $(v,t) \in \mathcal{V}$, $(u,s) \prec (v,t)$}{  \label{alg:wdpg:vp-ckeck2}
        $w \gets k_E(uv,st)$ \label{alg:wdpg:edge-comp} \;
        \If{$w > 0$}{
          create edge $e=(u,s)(v,t)$ \;
          $\mathcal{E} \gets \mathcal{E} \uplus \{e\}$ \;
          $w(e) = w$ \;
        }
      }\label{alg:wdpg:end-edge}
    }
  }
\end{algorithm}

Since the weighted direct product graph is undirected, we must avoid that the 
same pair of edges is processed twice. Therefore, we suppose that there is an 
arbitrary total order $\prec$ on the vertices $\mathcal{V}$, such that for every 
pair $(u,s),(v,t)\in \mathcal{V}$ either $(u,s) \prec (v,t)$ or 
$(v,t) \prec (u,s)$ holds.
In line~\ref{alg:wdpg:vp-ckeck2} we restrict the edge pairs that are compared to 
one of these cases.

\begin{proposition}\label{prop:runtime:wdpg}
 Let $n=\sizeV{G}$, $n'=\sizeV{H}$ and $m=\sizeE{G}$, $m'=\sizeE{H}$. 
 Algorithm~\ref{alg:wdpg} computes the weighted direct product graph in time
 $\bigO(n n' \mathsf{T}_{V} + m m' \mathsf{T}_{E})$,
 where $\mathsf{T}_{V}$ and $\mathsf{T}_{E}$ is the running time to compute vertex and edge
 kernels, respectively.
\end{proposition}

Note that in case of a sparse vertex kernel, which yields zero for most of the
vertex pairs of the input graph, $\sizeV{G \wdpg H} \ll \sizeV{G} \cdot \sizeV{H}$ 
holds. 
Algorithm~\ref{alg:wdpg} compares two edges by $k_E$ only in case of matching
endpoints (cf.\@ lines~\ref{alg:wdpg:vp-ckeck1}, \ref{alg:wdpg:vp-ckeck2}),
therefore in practice the running time to compare edges 
(line~\ref{alg:wdpg:start-edge}--\ref{alg:wdpg:end-edge}) might be considerably 
less than suggested by Proposition~\ref{prop:runtime:wdpg}. 
We show this empirically in Section~\ref{sec:evaluation:flrw}. In case of sparse graphs, i.e., 
$|E|=\bigO(|V|)$, and vertex and edge kernels which can be computed in time
$\bigO(1)$ the running time of Algorithm~\ref{alg:wdpg} is $\bigO(n^2)$, 
where $n=\max\{\sizeV{G},\sizeV{H}\}$.

\paragraph{Counting Weighted Walks.}\label{sec:walk_counting}
Given an undirected graph $G$ with adjacency matrix $\vec{A}$, let $a^\wlength_{ij}$ denote
the element at $(i,j)$ of the matrix $\vec{A}^\wlength$. It is well-known that $a^\wlength_{ij}$
is the number of walks from vertex $i$ to $j$ of length $\wlength$.
The number of $\wlength$-walks of $G$ consequently is 
$\sum_{i,j} a^\wlength_{i,j} = \mathbf{1}^\tp \vec{A}^\wlength \mathbf{1} = \mathbf{1}^\tp \vec{r}_\wlength$,
where $\vec{r}_\wlength = \vec{A} \vec{r}_{\wlength-1}$ with $\vec{r}_0 = \mathbf{1}$.
The $i$th element of the recursively defined vector $\vec{r}_\wlength$ is the number 
of walks of length $\wlength$ starting at vertex $i$.
Hence, we can compute the number of $\wlength$-walks by computing either matrix 
powers or matrix-vector products.
Note that even for sparse (connected) graphs $\vec{A}^\wlength$ quickly becomes 
dense with increasing walk length $\wlength$. 
The $\wlength$th power of an $n \times n$ matrix $\vec{A}$ can 
be computed na\"{\i}vely in time $\bigO(n^\omega \wlength)$ and $\bigO(n^\omega \log \wlength)$ 
using exponentiation by squaring, where $\omega$ is the exponent of matrix 
multiplication.
The vector $\vec{r}_\wlength$ can be computed by means of matrix-vector multiplications, 
where the matrix $\vec{A}$ remains unchanged over all iterations.
Since direct product graphs tend to be sparse in practice, we propose a method 
to compute the $\wlength$-walk kernel that is inspired by matrix-vector 
multiplication.

In order to compute the $\wlength$-walk kernel we do not want to count the walks, but
sum up the weights of each walk, which in turn are the product of vertex and edge 
weights.
Let $k_W$ be defined according to Equation~\eqref{eq:two-walk-kernel}, then we can 
formulate the $\wlength$-walk kernel as
\begin{equation}\label{eq:walk-kernel:recursive}
  K^=_\wlength(G, H) =
  \sum_{w \in \mathcal{W}_\wlength(G)} 
  \sum_{w' \in \mathcal{W}_\wlength(H)}
  k_W(w,w') = 
  \sum_{v \in V(G \wdpg H)} r_\wlength(v),
\end{equation}
where $r_\wlength$ is determined recursively according to
\begin{align}
r_i(u) \ = \ & \sum_{uv \in E(G \wdpg H)} w(u) \cdot w(uv) \cdot r_{i-1}(v) &&\forall u \in V(G \wdpg H) \label{eq:walk-kernel:recursive:i}\\
r_0(u) \ = \ & w(u) &&\forall u \in V(G \wdpg H). \label{eq:walk-kernel:recursive:0}
\end{align}
Note that $r_i$ can as well be formulated as matrix-vector product. 
We present a graph-based approach for computation akin to sparse matrix-vector 
multiplication, see Algorithm~\ref{alg:flrw}.

\begin{algorithm}
  \caption{Implicit computation of $\wlength$-walk kernel}
  \label{alg:flrw}
  \Input{Graphs $G$, $H$, kernels $k_V$, $k_E$ and length parameter $\wlength$.}
  \Output{Value $K^=_\wlength(G,H)$ of the $\wlength$-walk kernel.}
  \BlankLine

  $(\mathcal{V}, \mathcal{E}, w) \gets \WDPG(G,H,k_V,k_E)$ \note*[r]{Compute $G \wdpg H$}
  \ForAll{$v \in \mathcal{V}$}{
    $r_0(v) \gets w(v)$ \note*[r]{Initialization}
  }
  \For{$i \gets 1$ \KwTo $\wlength$}{
    \ForAll{$u \in \mathcal{V}$}{
      $r_i(u) \gets 0$\;
      \ForAll(\note*[f]{Neighbors of $u$ in $G \wdpg H$}){$v \in \N(u)$}{ 
        $r_i(u) \gets r_i(u) + w(u) \cdot w(uv) \cdot r_{i-1}(v)$ \;
      }
    } 
  }
  \Return $\sum_{v \in \mathcal{V}} r_\wlength(v)$ \label{alg:flrw:return} \;
\end{algorithm}

\begin{theorem}
 Let $n=|\mathcal{V}|$, $m=|\mathcal{E}|$. Algorithm~\ref{alg:flrw} computes the 
 $\wlength$-walk kernel in time $\bigO(n+\wlength(n+m) + \mathsf{T}_{\WDPG})$, where $\mathsf{T}_{\WDPG}$ is the
 time to compute the weighted direct product graph.
\end{theorem}
Note that the running time depends on the size of the product graph and 
$n \ll \sizeV{G} \cdot \sizeV{H}$ and $m \ll \sizeE{G} \cdot \sizeE{H}$ is possible 
as discussed in Section~\ref{sec:product-graph}.

The Max-$\wlength$-walk kernel is the sum of the $j$-walk kernels with $j \leq \wlength$ and,
hence, with Equation~\eqref{eq:walk-kernel:recursive} we can also formulate it 
recursively as
\begin{equation}
 K^\leq_\wlength (G,H) = \sum_{i=0}^\wlength \lambda_i K^=_i(G, H) = 
 \sum_{i=0}^\wlength \lambda_i \sum_{v \in V(G \wdpg H)} r_i(v).
\end{equation}
This value can be obtained from Algorithm~\ref{alg:flrw} by simply changing the 
return statement in line~\ref{alg:flrw:return} according to the right-hand side
of the equation without affecting the asymptotic running time.

\subsubsection{Explicit kernel computation}\label{sec:flrw:kernel-computation:explicit}
We have shown in Section~\ref{sec:walk_r-convolution_kernel} that $\wlength$-walk kernels 
are $R$-convolution kernels. Therefore, we can derive explicit feature maps with the 
techniques introduced in Section~\ref{sec:maps}.
Provided that we know explicit feature maps for the vertex and edge kernel, we can derive
explicit feature maps for the kernel on walks and obtain an explicit computation scheme 
by enumerating all walks.
We propose a more elaborated approach that avoids enumeration and exploits the simple 
composition of walks.

To this end, we consider Equations~\eqref{eq:walk-kernel:recursive}, 
\eqref{eq:walk-kernel:recursive:i}, \eqref{eq:walk-kernel:recursive:0} of the recursive 
product graph based implicit computation. Let us assume that $\phi^V$ and $\phi^E$ are 
feature maps of the kernels $k_V$ and $k_E$, respectively. Then we can derive a feature 
map $\Phi_i$, such that $\langle \Phi_i(u), \Phi_i(u') \rangle = r_i(u,u')$ for all
$(u,u') \in V(G) \times V(H)$\footnote{We assume $r_i(u,u')=0$ for 
$(u,u') \notin V(G \wdpg H)$ .} as follows using the results of Section~\ref{sec:maps}.

\begin{align*}
 \Phi_0(u) \ = \ & \phi^V(u) &&\forall u \in V(G) \\
 \Phi_i(u) \ = \ & \sum_{uv \in E(G)} \phi^V(u) \otimes \phi^E(uv) \otimes \Phi_{i-1}(v) &&\forall u \in V(G)
\end{align*}
From these, the feature map $\phi^=_\wlength$ of the $\wlength$-walk kernel is obtained 
according to
\begin{equation*}
 \phi^=_\wlength(G) = \sum_{v \in V(G)} \Phi_\wlength(v).
\end{equation*}
Algorithm~\ref{alg:flrw_explicit} provides the pseudo code of this computation.
We can easily derive a feature map of the Max-$\wlength$-walk kernel from the 
feature maps of all $i$-walk kernels with $i \leq \wlength$, cf.\@ Proposition~\ref{prop:add-mult-maps}.

\begin{algorithm}
  \caption{Generating feature vectors of the $\wlength$-walk kernel}
  \label{alg:flrw_explicit}
  \Input{Graph $G$, length parameter $\wlength$, feature map $\phi^V$ of $k_V$ and $\phi^E$ of $k_E$.}
  \Output{Feature vector $\phi^=_\wlength(G)$ of the $\wlength$-walk kernel.}
  \Data{Feature vectors $\Phi_i(v)$ encoding the contribution of $i$-walks starting at $v$.}
  \BlankLine
  
  \ForAll{$v \in V(G)$}{
    $\Phi_0(v) \gets \phi^V(v)$ \note*[r]{Initialization, length $0$ walks}
  }
  \For{$i \gets 1$ \KwTo $\wlength$}{
    \ForAll{$u \in V(G)$}{
      $\Phi_i(u) \gets \vec{0}$ \;
      \ForAll{$v \in \N(u)$}{
        $\Phi_i(u) \gets \Phi_i(u) + \phi^V(u) \otimes \phi^E(uv) \otimes \Phi_{i-1}(v)$ \;
      }
    } 
  }
  \Return $\sum_{v \in V(G)} \Phi_\wlength(v)$ \note*[r]{Combine vectors}
\end{algorithm}

The dimension of the feature space and the density of feature vectors depends 
multiplicative on the same properties of the feature vectors of $k_V$ and $k_E$. 
For non-trivial vertex and edge kernels explicit computation of the $\wlength$-walk 
kernel is likely to be infeasible in practice.
Therefore, we now consider graphs with simple labels from the alphabet $\mathcal{L}$ and 
the kernel $k^\delta_W$ given by Equation~\eqref{eq:two-walk-kernel-dirac}. 
Following \citet{Gaertner2003} we can construct a feature map in this case, where the the
features are sequences of labels associated with walks. As we will see later, this 
feature space is indeed obtained with Algorithm~\ref{alg:flrw_explicit}.
A walk $w$ of length $\wlength$ is associated with a label sequence 
$\lab(w)=(\lab(v_0),\lab(e_1),\dots,\lab(v_\wlength)) \in \mathcal{L}^{2\wlength+1}$.
Moreover, graphs are decomposed into walks and two walks $w$ and $w'$ are 
considered equivalent if and only if $\lab(w) = \lab(w')$.
This gives rise to the feature map $\phi^=_\wlength$, where each component is
associated with a label sequence $s \in \mathcal{L}^{2\wlength+1}$ and counts
the number of walks $w \in \mathcal{W}_\wlength(G)$ with $\lab(w) = s$.
Note that the obtained feature vectors have $|\mathcal{L}|^{2\wlength+1}$ 
components, but are typically sparse.
In fact, Algorithm~\ref{alg:flrw_explicit} constructs this feature map.
We assume that $\phi^V(v)$ and $\phi^E(uv)$ have exactly one non-zero component 
associated with the label $\lab(v)$ and $\lab(uv)$, respectively.
Then the single non-zero component of $\phi^V(u) \otimes \phi^E(uv) \otimes \phi^V(v)$
is associated with the label sequence $\lab(w)$ of the walk $w=(u, uv, v)$.
A walk of length $\wlength$ can be decomposed into a walk of length $\wlength-1$
with an additional edge and vertex added at the front. This allows to obtain the 
number of walks of length $\wlength$ with a given label sequence starting at a 
fixed vertex $v$ by concatenating $(\lab(v),\lab(vu))$ with all label sequences 
for walks starting from a neighbor $u$ of the vertex $v$.
This construction is applied in every iteration of the outer for-loop in Algorithm~\ref{alg:flrw_explicit} and the feature vectors $\Phi_i$ are easy to 
interpret. 
Each component of $\Phi_i(v)$ is associated with a label sequence $s \in \mathcal{L}^{2i+1}$
and counts the walks $w$ of length $i$ starting at $v$ with $\lab(w)=s$.

We consider the running time for the case, where graphs have discrete labels and
the kernel $k^\delta_W$ is given by Equation~\eqref{eq:two-walk-kernel-dirac}. 

\begin{theorem}
 Given a graph $G$ with $n=\sizeV{G}$ vertices and $m=\sizeE{G}$ edges, 
 Algorithm~\ref{alg:flrw_explicit} computes the $\wlength$-walk kernel feature 
 vector $\phi^=_\wlength(G)$ in time $\bigO(n+\wlength(n+m)s)$, where $s$ is the 
 maximum number of different label sequences of $(\wlength-1)$-walks staring at 
 a vertex of $G$.
\end{theorem}

Assume Algorithm~\ref{alg:flrw_explicit} is applied to unlabeled sparse graphs,
i.e., $\sizeE{G} = \bigO(\sizeV{G})$, then $s = 1$ and the feature mapping can be 
performed in time $\bigO(n+\wlength n)$.
This yields a total running time of $\bigO(d \wlength n + d^2)$ to compute a 
kernel matrix for $d$ graphs of order $n$ for $\wlength>0$.

\subsection{Shortest-path kernel}\label{sec:shortes-path:exp-impl}
A classical kernel applicable to attributed graphs is the shortest-path kernel 
\citep{Borgwardt2005}. This kernel compares all shortest paths in two graphs 
according to their lengths and the vertex annotation of their endpoints. 
The kernel was proposed with an implicit computation scheme, but explicit 
methods of computation have been reported to be used for graphs with discrete 
labels.

The shortest-path kernel is defined as
\begin{equation}\label{eq:sp}
  k^{\text{SP}}(G,H) = \hspace{-.5em} \sum_{\substack{u,v \in V(G),\\ u \neq v}} \sum_{\substack{w,z \in V(H),\\ w \neq z}} \hspace{-.5em} k_V(u,w) \cdot k_E(d_{uv},d_{wz})\cdot k_V(v,z),
\end{equation}
where $k_V$ is a kernel comparing vertex labels of the respective starting and 
end vertices of the paths. Here, $d_{uv}$ denotes the length of a shortest path from
$u$ to $v$ and $k_E$ is a kernel comparing path lengths with
$k_E(d_{uv}, d_{wz}) = 0$ if  $d_{uv} = \infty$ or $d_{wz} = \infty$.

Its computation is performed in two steps~\citep{Borgwardt2005}: 
for each graph $G$ of the data set the complete graph $G'$ on the vertex set $V(G)$
is generated, where an edge $uv$ is annotated with the length of a shortest path 
from $u$ to $v$. The shortest-path kernel then is equivalent to the walk kernel 
with fixed length $\wlength=1$ between these transformed graphs, where the kernel 
essentially compares all pairs of edges.
The kernel $k_E$ used to compare path lengths may, for example, be realized by the 
Brownian Bridge kernel~\citep{Borgwardt2005}.

For the application to graphs with discrete labels a more efficient method of 
computation by explicit mapping has been reported by~\citet[Section 3.4.1]{Shervashidze2011}.
When $k_V$ and $k_E$ both are Dirac kernels, each component of the feature vector
corresponds to a triple consisting of two vertex labels and a path length. This 
method of computation has been applied in several experimental comparisons, e.g., 
\citep{Kriege2012,Morris2016}. This feature map is directly obtained from our 
results in Section~\ref{sec:maps}. It is as well rediscovered from our explicit 
computation schemes for fixed length walk kernels reported in Section~\ref{sec:flrw}.
However, we can also derive explicit feature maps for non-trivial kernels 
$k_V$ and $k_E$. Then the dimension of the feature map increases due to the 
product of kernels, cf.\@ Equation~\ref{eq:sp}. We will study this and the effect 
on running time experimentally in Section~\ref{sec:evaluation}.

\subsection{Graphlet, subgraph and subgraph matching kernels}\label{sec:subgraph:exp-impl}
Subgraph or graphlet kernels have been proposed for unlabeled graphs or graphs
with discrete labels~\citep{Gaertner2003,Wale2008a,Shervashidze2009}. 
The subgraph matching kernel has been developed as an extension for attributed 
graphs~\citep{Kriege2012}.

Given two graphs $G$ and $H$ in $\mathcal{G}$, the \emph{subgraph kernel} is 
defined as
\begin{equation}\label{eq:kernel:subgraph}
k^\subseteq(G, H) = \sum_{G' \subseteq G} \sum_{H' \subseteq H} k_\simeq(G',H'),
\end{equation}
where $k_\simeq \colon \mathcal{G} \times \mathcal{G} \rightarrow \{0,1\}$ is the
isomorphism kernel, i.e., $k_\simeq(G',H')=1$ if and only if $G'$ and $H'$ are isomorphic.
A similar kernel was defined by~\citet{Gaertner2003} and its computation was shown
to be $\mathsf{NP}$-hard. However, it is polynomial time computable when 
considering only subgraphs up to a fixed size. 
The subgraph kernel, cf.\@ Equation~\eqref{eq:kernel:subgraph}, is easily 
identified as an instance of the crossproduct kernel, cf.\@ Equation~\eqref{eq:kernel:crossproduct}. The base kernel $k_\simeq$ is not the 
trivial Dirac kernel, but binary, cf.\@ Section~\ref{sec:maps:dirac}.
The equivalence classes induced by $k_\simeq$ are referred to as isomorphism classes
and distinguish subgraphs up to isomorphism. The feature map $\phi_\subseteq$
of $k^\subseteq$ maps a graph to a vector, where each component counts the number 
of occurrences of a specific graph as subgraph in $G$.
Determining the isomorphism class of a graph is known as \emph{graph canonization 
problem} and well-studied. By solving the graph canonization problem instead of the 
graph isomorphism problem we obtain an explicit feature map for the subgraph kernel.
Although graph canonization clearly is at least as hard as graph isomorphism, 
the number of canonizations required is linear in the number of
subgraphs, while a quadratic number of isomorphism tests would be required for
a single na\"{\i}ve computation of the kernel. The gap in terms of runtime even 
increases when computing a whole kernel matrix, cf.\@ Section~\ref{sec:preliminary}.

Indeed, the observations above are key to several graph kernels.
The graphlet kernel~\citep{Shervashidze2009}, also see Section~\ref{sec:related-work}, 
is an instance of the subgraph kernel and computed by explicit feature maps. 
However, only unlabeled graphs of small size are considered by the graphlet kernel, 
such that the canonizing function can be computed easily. The same approach was 
taken by~\citet{Wale2008a} considering larger connected subgraphs of labeled 
graphs derived from chemical compounds. On the contrary, for attributed graphs 
with continuous vertex labels, the function $k_\simeq$ is not sufficient to 
compare subgraphs adequately.
Therefore, subgraph matching kernels were proposed by \citet{Kriege2012}, which 
allow to specify arbitrary kernel functions to compare vertex and edge 
attributes.
Essentially, this kernel considers all mappings between subgraphs and scores
each mapping by the product of vertex and edge kernel values of the vertex
and edge pairs involved in the mapping. When the specified vertex and edge 
kernels are Dirac kernels, the subgraph matching kernel is equal to the subgraph
kernel up to a factor taking the number of automorphisms between subgraphs into 
account~\citep{Kriege2012}. 
Based on the above observations explicit mapping of subgraph matching kernels 
is likely to be more efficient when subgraphs can be adequately compared by a
binary kernel.

\subsection{Discussion}
A crucial observation of our study of feature maps for composed kernels in 
Section~\ref{sec:maps} is that the number of components of the feature vectors 
increases multiplicative under taking products of kernels; this also holds in 
terms of non-zero components. 
Unless feature vectors have few non-zero components, this operation is likely
to be prohibitive in practice. However, if feature vectors have exactly one 
non-zero component like those associated with binary kernels, taking products of 
kernels is manageable by sparse data structures.

In this section we have systematically constructed and discussed feature maps
of several graph kernels and the observation mentioned above is expected to 
affect the kernels to varying extents.
While weighted vertex kernels do not take products of vertex and edge kernels,
the shortest-path kernel and, in particular, the subgraph matching and fixed 
length walk kernels heavily rely on multiplicative combinations.
Considering the relevant special case of a Dirac kernel, which leads to feature
vectors with only one non-zero component, the rapid growth due to multiplication 
is tamed. In this case the number of substructures considered as different 
according to the vertex and edge kernels determines the number of non-zero 
components of the feature vectors associated with the graph kernel.
The basic characteristics of the considered graph kernels are summarized in 
Table~\ref{tab:kernel-map}.
The sparsity of the feature vectors of the vertex and edge kernels is an important 
intervening factor, which is difficult to assess theoretically and we proceed
by an experimental study.

\begin{table*}
\small
\setlength{\tabcolsep}{10pt}
\def\arraystretch{1.3}
\begin{center}
  \caption{Graph kernels and their properties. We consider graphs on $n$ vertices
   and $m$ edges with maximum degree $\Delta$. 
   Let $\delta$ be the maximum diameter of any graph $G$ in the data sets and $C$ 
   the total number of colors appearing in $h$ iterations of Weisfeiler-Lehman 
   refinement.
   The dimension of the feature space associated with $k_V$ and $k_E$ is denoted 
   by $d_V$ and $d_E$, respectively, while $\mathsf{T}_{V}$ and $\mathsf{T}_{E}$
   is the time to evaluate the vertex and edge kernel once. The parameters $\wlength$
   and $s$ denote the walk length and subgraph size, respectively.}
  \label{tab:kernel-map}
  \begin{tabular}{@{}lccc@{}}\toprule
    \textbf{Graph kernel}          & \textbf{Parts}           & \textbf{Dimension}       & \textbf{Running time (implicit)}                                                 \\\midrule
    \textsc{GraphHopper}           & $\bigO(n)$               & $\delta^2 d_V$           & $\bigO\left(n^2(m + \log n + \mathsf{T}_{V} + \delta^2 )\right)$ \\
    \textsc{GraphInvariant}        & $\bigO(n)$               & $C d_V$                  & $\bigO\left(hm + n^2 \mathsf{T}_{V}\right)$                      \\
    \textsc{FixedLengthWalk} & $\bigO(\Delta^\wlength)$ & $d_V+(d_V d_E)^\wlength$ & $\bigO\left(\wlength(n^2+m^2) + n^2\mathsf{T}_{V}+m^2\mathsf{T}_{E}\right)$         \\
    \textsc{ShortestPath}          & $\bigO(n^2)$             & $d_V^2 d_E$              & $\bigO\left(n^2\mathsf{T}_{V}+n^4\mathsf{T}_{E}\right)$          \\
    \textsc{SubgraphMatching}      & $\bigO(n^s)$             & $(d_V d_E^2)^s s!$       & $\bigO\left(sn^{2s+2} + n^2\mathsf{T}_{V}+n^4\mathsf{T}_{E} \right)$                                   \\
    \bottomrule
  \end{tabular}
\end{center}
\end{table*}

\section{Experimental evaluation} \label{sec:evaluation}
Our goal in this section is to answer the following questions experimentally.
\begin{enumerate}[leftmargin=*,labelindent=0pt,label=\bfseries Q\arabic*]
  \item Are approximative explicit feature maps of kernels for attributed graphs 
        competitive in terms of running time and classification accuracy compared
        to exact implicit computation?
  \item Are exact explicit feature maps competitive for kernels relying on
        multiplication when the Dirac kernel is used to compare discrete labels?
        How do the graph properties such as label diversity affect the running 
        time?
        \begin{enumerate}[label=\bfseries \alph*)]
          \item How does the fixed length walk kernel behave with regard to these 
                questions and what influence does the walk length have?
          \item Can the same behavior regarding running time be observed for the 
                graphlet and subgraph matching kernel?
        \end{enumerate}
\end{enumerate}

\subsection{Experimental setup}
All algorithms were implemented in Java and the default Java \texttt{HashMap} class
was used to store feature vectors.
Due to the varied memory requirements of the individual series of experiments, 
different hardware platforms were used in Sections~\ref{sec:evaluation:attr},~\ref{sec:evaluation:flrw}
and~\ref{sec:evaluation:sm}.
In order to compare the running time of both computational strategies 
systematically without the dependence on one specific kernel method, we report
the running time to compute the quadratic kernel matrices, unless stated otherwise.
We performed classification experiments using the $C$-SVM implementation 
LIBSVM \citep{Chang2011}.
We report mean prediction accuracies obtained by $10$-fold cross-validation 
repeated $10$ times with random fold assignments.
Within each fold all necessary parameters were selected by cross-validation 
based on the training set. This includes the regularization parameter $C$ 
selected from $\{10^{-3}, 10^{-2}, \dots, 10^3\}$, all kernel parameters, where 
applicable, and whether to normalize the kernel matrix.

\subsubsection{Data sets}
We performed experiments on synthetic and real-world data sets from different 
domains, see Table~\ref{tab:datasets} for an overview on their characteristics. 
All data sets can be obtained from our publicly available collection~\citep{KKMMN2016}
unless the source is explicitly stated.

\paragraph{Small molecules.}
Molecules can naturally be represented by graphs, where vertices represent atoms 
and edges represent chemical bonds. \textsc{Mutag} is a data set of chemical 
compounds divided into two classes according to their mutagenic effect on a 
bacterium. This small data set is commonly used in the graph kernel literature.
In addition we considered the larger data set \textsc{U251}, which stems from 
the NCI Open Database provided by the National Cancer Institute (NCI). 
In this data set the class labels indicate the ability of a compound to inhibit 
the growth of the tumor cell line U251. We used the data set processed by 
\citet{Swamidass2005}, which is publicly available from the ChemDB 
website.\footnote{\url{http://cdb.ics.uci.edu}}

\paragraph{Macromolecules.}
\textsc{Enzymes} and \textsc{Proteins} both represent macromolecular structures
and were obtained from \citep{Borgwardt2005a, Feragen2013}. The following graph
model has been employed.
Vertices represent secondary structure elements (SSE) and are annotated
by their type, i.e., helix, sheet or turn, and a rich set of physical and 
chemical attributes. Two vertices are connected by an edge if they are neighbors 
along the amino acid sequence or one of three nearest neighbors in space. 
Edges are annotated with their type, i.e., structural or sequential.
In \textsc{Enzymes} each graph is annotated by an EC top level class, which 
reflects the chemical reaction the enzyme catalyzes, \textsc{Proteins} is divided 
into enzymes and non-enzymes.

\paragraph{Synthetic graphs.}
The data sets \textsc{SyntheticNew} and \textsc{Synthie} were synthetically 
generated to obtain classification benchmarks for graph kernels with attributes.
We refer the reader to the publications~\citep{Feragen2013}\footnote{We used the 
updated version of the data set \textsc{Synthetic} published together with the 
Erratum to \citep{Feragen2013}.} and~\citep{Morris2016}, respectively, for the 
details of the generation process. Additionally, we generated new synthetic graphs
in order to systematically vary graph properties of interest like the label 
diversity, which we expect to have an effect on the running time according to 
our theoretical analysis.

\begin{table*}
\setlength{\tabcolsep}{4.5pt}
		\begin{center}
		\ra{1.0}
		\caption{Data set statistics and properties.}
			\begin{tabular}{@{}lrcrrcc@{}}\toprule
			\multirow{3}{*}{\vspace*{8pt}\textbf{Data set}}&\multicolumn{6}{c}{\textbf{Properties}}\\\cmidrule{2-7}
			&  Graphs& Classes  & Avg. $|V|$ & Avg. $|E|$ & Vertex/edge labels & Attributes  \\\midrule
			\textsc{Mutag}  & 188 & 2 & 17.9 & 19.8 & $+$/$+$ & $-$ \\
			\textsc{U251}  & 3\,755 & 2 & 23.1 & 24.8 & $+$/$+$ & $-$ \\
			\textsc{Enzymes} & 600 & 6 & 32.6 & 62.1 & $+$/$+$ & 18 \\
			\textsc{Proteins} & 1\,113 & 2 & 39.1 & 72.8 & $+$/$-$ & 1 \\
			\textsc{SyntheticNew}  & 300 & 2 & 100.0& 196.3& $-$/$-$ & 1 \\
			\textsc{Synthie}   & 400 & 4 & 95.0 & 172.9 & $-$/$-$ & 15 \\
			\bottomrule
		\end{tabular}
		\label{tab:datasets}
	\end{center}
\end{table*}

\subsection{Approximative explicit feature maps of kernels for attributed graphs (Q1)}\label{sec:evaluation:attr}
We have derived explicit computation schemes of kernels for attributed
graphs, which have been proposed with an implicit method of computation.
Approximative explicit computation is possible under the assumption that the 
kernel for the attributes can be approximated by explicit feature maps.
We compare both methods of computation w.r.t.\@ their running time and the 
obtained classification accuracy on the four attributed graph data sets
\textsc{Enzymes}, \textsc{Proteins}, \textsc{SyntheticNew} and \textsc{Synthie}.
Since the discrete labels alone are often highly informative, we ignored discrete 
labels if present and considered the real-valued vertex annotations only in order 
to obtain challenging classification problems.
All attributes where dimension-wise linearly scaled to the range $[0,1]$ in a 
preprocessing step.

\subsubsection{Method}
We employed three kernels for attributed graphs: the shortest-path kernel, cf.
Section~\ref{sec:shortes-path:exp-impl}, the \textsc{GraphHopper} and 
\textsc{GraphInvariant} kernel as described in Section~\ref{sec:wv}.
Preliminary experiments with the subgraph matching kernel showed that it cannot 
be computed by explicit feature maps for non-trivial subgraph sizes due to 
its high memory consumption. 
The same holds for fixed length walk kernels with walk length $\wlength \geq 3$.
Therefore, we did not consider these kernels any further regarding Q1, but 
investigate them for graphs with discrete labels in Sections~\ref{sec:evaluation:flrw}
and~\ref{sec:evaluation:sm} to answer Q2a and Q2b.

For the shortest-path kernel we used the Dirac kernel to compare path lengths 
and selected the number of Weisfeiler-Lehman refinement steps for the 
\textsc{GraphInvariant} kernel from $h \in \{0, \dots, 7\}$.
For the comparison of attributes we employed the dimension-wise product of the hat kernel $k_\Delta$ as defined 
in Equation~\eqref{eq:rbf_hat} choosing $\delta$ from $\{0.2, 0.4, 0.6, 0.8, 1.0, 1.5, 2.0\}$.
The three kernels were computed functionally employing this kernel as a base line.
We obtained approximate explicit feature maps for the attribute kernel by the
method of~\citet{Rahimi2008} and used these to derive approximate explicit feature maps 
for the graph kernels. We varied the number of random binning features from 
$\{1, 2, 4, 8, 16, 32, 64\}$, which corresponds to the number of non-zero 
components of the feature vectors for the attribute kernel and controls its
approximation quality.
Please note that the running time is effected by the kernel parameters, i.e., 
$\delta$ of Equation~\eqref{eq:rbf_hat} and the number $h$ of Weisfeiler-Lehman 
refinement steps for \textsc{GraphInvariant}.
Therefore, in the following we report the running times for fixed values 
$\delta=1$ and $h=3$, which were selected frequently by cross-validation.

All experiments were conducted using Oracle Java v1.8.0 on an Intel Xeon 
E5-2640 CPU at 2.5GHz with 64GB of RAM using a single processor only.

\begin{figure}
   \centering
   \subfigure[\textsc{Enzymes}, \textsc{GH}]{\label{fig:im_ex:enzymes:gh}
      \includegraphics[scale=.87]{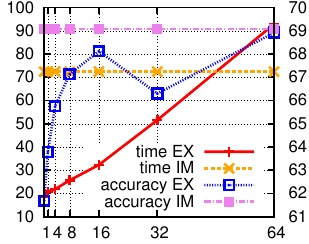}
   } \hfill
   \subfigure[\textsc{Enzymes}, \textsc{GI}]{\label{fig:im_ex:enzymes:wvwl}
      \includegraphics[scale=.87]{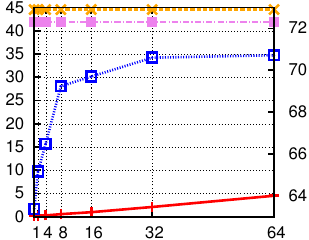}
   } \hfill
   \subfigure[\textsc{Enzymes}, \textsc{SP}]{\label{fig:im_ex:enzymes:sp}
      \includegraphics[scale=.87]{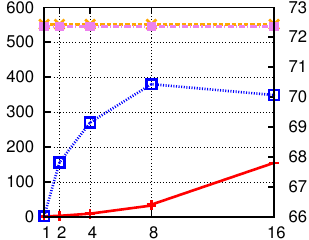}
   } \hfill
   \subfigure[\textsc{Synthie}, \textsc{GH}]{\label{fig:im_ex:synthie:gh}
      \includegraphics[scale=.87]{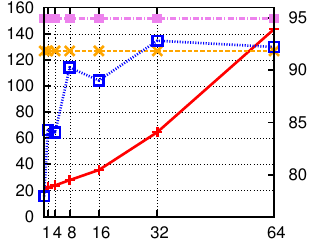}
   } \hfill
   \subfigure[\textsc{Synthie}, \textsc{GI}]{\label{fig:im_ex:synthie:wvwl}
      \includegraphics[scale=.87]{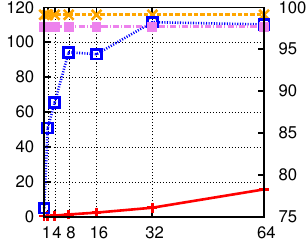}
   } \hfill
   \subfigure[\textsc{Synthie}, \textsc{SP}]{\label{fig:im_ex:synthie:sp}
      \includegraphics[scale=.87]{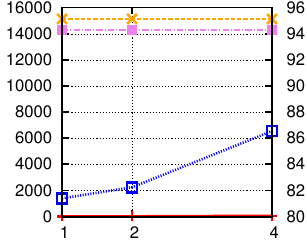}
   } \hfill
   \subfigure[\textsc{Proteins}, \textsc{GH}]{\label{fig:im_ex:proteins:gh}
      \includegraphics[scale=.87]{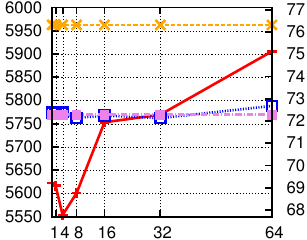}
   } \hfill
   \subfigure[\textsc{Proteins}, \textsc{GI}]{\label{fig:im_ex:proteins:wvwl}
      \includegraphics[scale=.87]{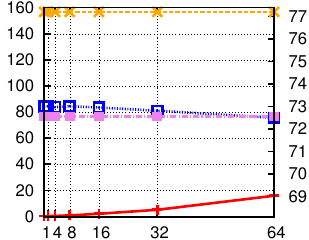}
   } \hfill
   \subfigure[\textsc{Proteins}, \textsc{SP}]{\label{fig:im_ex:proteins:sp}
      \includegraphics[scale=.87]{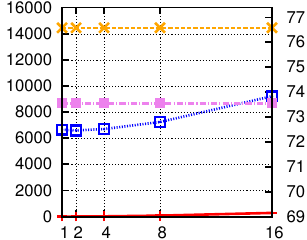}
   } \hfill
   \subfigure[\textsc{SyntheticNew}, \textsc{GH}]{\label{fig:im_ex:synthetic:gh}
      \includegraphics[scale=.87]{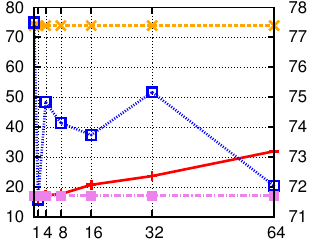}
   } \hfill
   \subfigure[\textsc{SyntheticNew}, \textsc{GI}]{\label{fig:im_ex:synthetic:wvwl}
      \includegraphics[scale=.87]{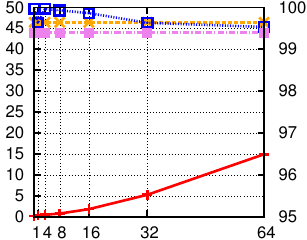}
   } \hfill
   \subfigure[\textsc{SyntheticNew}, \textsc{SP}]{\label{fig:im_ex:synthetic:sp}
      \includegraphics[scale=.87]{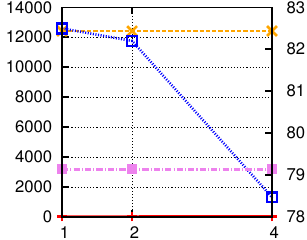}
   }
   \caption{Running times and classification accuracies of graph kernels 
   approximated by explicit feature maps with $2^i$, $i \in \{0,\dots,4\}$,
   iterations of random binning. The results of exact implicit computation are
   shown as a base line   
   (left $y$-axes shows the running time in seconds, right $y$-axes the accuracy 
   in percent; GH --- \textsc{GraphHopper}, GI --- \textsc{GraphInvariant}, 
   SP --- \textsc{ShortestPath}).
   } 
   \label{fig:im_ex}
\end{figure}

\subsubsection{Results and discussion}
We were not able to compute the shortest-path kernel by explicit feature maps 
with more than 16 iterations of binning for the base kernel on \textsc{Enzymes}
and \textsc{Proteins} and no more than 4 iterations on \textsc{Synthie} and
\textsc{SyntheticNew} with 64 GB of main memory. The high memory consumption
of this kernel is in accordance with our theoretical analysis, since 
the multiplication of vertex and edge kernels drastically increases the number
of non-zero components of the feature vectors. This problem does not effect the 
two weighted vertex kernels to the same extent.
We observed the general trend that the memory consumption and running time
increases with small values of $\delta$. This is explained by the fact that 
the number of components of the feature vectors of the vertex kernels increases 
in this case. Although the number of non-zero components does not increase for 
these feature vectors, it does for the graph kernel feature vectors, since the number 
of vertices with attributes falling into different bins increases.

The results on running time and accuracy are summarized in Figure~\ref{fig:im_ex}. 
For the two data sets \textsc{Enzymes} and \textsc{Synthie} we observe that the 
classification accuracy obtained by the approximate explicit feature maps 
approaches the accuracy obtained by the exact method with increasing number of 
binning iterations. For the other two data sets the accuracy does not improve
with the number of iterations. For \textsc{Proteins} the kernels obtained with a 
single iteration of binning, i.e., essentially applying a Dirac kernel, achieve
an accuracy at the same level as the exact kernel obtained by implicit computation. 
This suggests that for this data set a trivial comparison of attributes is sufficient 
or that the attributes are not essential for classification at all.
For \textsc{SyntheticNew} the kernels using a single iteration of binning are even 
better than the exact kernel, but get worse as the number of iterations increases.
One possible explanation for this is that the vertex kernel used is not a good 
choice for this data set.

With few iteration of binning the explicit computation scheme is always faster
than the implicit computation. The growth in running time with increasing number 
of binning iterations for the vertex kernel varies between the graph 
kernels. Approximating the \textsc{GraphHopper} kernel by explicit feature maps
with 64 binning iteration for the vertex kernel leads to a running time similar
to the one required for its exact implicit computation on all data sets with 
exception of \textsc{SyntheticNew}. On this data set explicit computation remains
faster.
For \textsc{GraphInvariant} explicit feature maps lead to a running time which is orders
of magnitude lower than implicit computation. Although both, \textsc{GraphHopper}
and \textsc{GraphInvariant} are weighted vertex kernels, this difference can be explained
by the number of non-zero components in the feature vectors of the weight kernel.
We observe that \textsc{GraphInvariant} clearly provides the best classification accuracy 
for two of the four data sets and is competitive for the other two. At the same
time \textsc{GraphInvariant} can be approximated very efficiently by explicit feature maps.
Therefore, even for attributed graphs effective and efficient graph kernels can 
be obtained from explicit feature maps by our approach.

\subsection{Kernels for graphs with discrete labels (Q2)}\label{sec:evaluation:overall}
In order to get a first impression of the runtime behavior of explicit and implicit
computation schemes on graphs with discrete labels, we computed the kernel matrix
for the standard data sets ignoring the attributes, if present.
The experiments were conducted using Java OpenJDK v1.7.0 on an Intel Core i7-3770 CPU 
at 3.4GHz (Turbo Boost disabled) with 16GB of RAM using a single processor only.
The reported running times are average values over $5$ runs.

The results are summarized in Table~\ref{tab:times_discrete}.
For the shortest-path kernel explicit mapping clearly outperforms implicit 
computation by several orders of magnitude with respect to running time.
This is in accordance with our theoretical analysis and our results suggest to 
always use explicit computation schemes for this kernel whenever a Dirac kernel
is adequate for label and path length comparison. In this case memory consumption
is unproblematic, in contrast to the setting discussed in 
Section~\ref{sec:evaluation:attr}.

Note that the explicit computation of the fixed length random walk kernel is
extremely efficient on \textsc{SyntheticNew} and \textsc{Synthie}, which have 
uniform discrete labels, but strongly depends on the walk length for the other 
data sets. This obsevation is investigated in detail in Section~\ref{sec:evaluation:flrw}.
The running times of the connected subgraph matching kernel and the graphlet 
kernel are studied exhaustively in Section~\ref{sec:evaluation:sm}.

\newcommand*{\thour}{\text{h\,}}
\newcommand*{\tmin}{\ensuremath{^{\prime}\mkern-1.2mu}}
\newcommand*{\tsec}{\ensuremath{^{\prime\prime}\mkern-1.2mu}}
\begin{table*}
\setlength{\tabcolsep}{4.5pt}
		\begin{center}
		\ra{1.0}
		\caption{Running times of the fixed length walk kernel (FLRW$_\wlength$) with walk lengths $\wlength$, the shortest-path kernel (SP), connected subgraph matching kernel (CSM) and the graphlet kernel (GL) on graphs with discrete labels in seconds unless stated otherwise. OOM --- out of memory.}
			\begin{tabular}{@{}lrcccccc@{}}\toprule
			\multicolumn{2}{c}{\textbf{Kernel}} & \textsc{Mutag} & \textsc{U251} & \textsc{Enzymes} & \textsc{Proteins} & \textsc{SyntheticNew} & \textsc{Synthie} \\\midrule
			\multirow{11}{*}{\rotatebox{90}{\textbf{Implicit}}}
         & FLRW$_0$  & 0.618 & 250.3 & 20.67 & 100.8 & 159.5 & 224.2 \\
         & FLRW$_1$  & 0.606 & 281.6 & 23.45 & 116.4 & 202.3 & 284.1 \\
         & FLRW$_2$  & 0.652 & 303.4 & 26.03 & 132.0 & 236.2 & 330.7 \\
         & FLRW$_3$  & 0.617 & 323.8 & 28.18 & 143.0 & 270.4 & 377.1 \\
         & FLRW$_4$  & 0.653 & 343.8 & 30.63 & 156.2 & 304.3 & 424.2 \\
         & FLRW$_5$  & 0.693 & 363.7 & 32.65 & 169.5 & 336.9 & 468.6 \\
         & FLRW$_6$  & 0.733 & 383.5 & 34.86 & 182.5 & 371.2 & 513.6 \\
         & FLRW$_7$  & 0.779 & 404.4 & 36.94 & 195.8 & 404.4 & 558.9 \\
         & FLRW$_8$  & 0.870 & 425.3 & 38.16 & 208.8 & 438.0 & 603.9 \\
         & FLRW$_9$  & 0.877 & 447.7 & 39.97 & 221.3 & 470.1 & 648.1 \\
         & SP        & 5.272 & 2\thour6\tmin55\tsec & 9\tmin8\tsec & 4\thour58\tmin40\tsec & 2\thour14\tmin23\tsec & 2\thour39\tmin30\tsec \\
         & CSM        & 15.45 & 4\thour0\tmin16\tsec & 34\tmin15\tsec & OOM & $>$24\thour & $>$24\thour \\\midrule
			\multirow{11}{*}{\rotatebox{90}{\textbf{Explicit}}}
         & FLRW$_0$  & 0.004 & 0.868 & 0.029 & 0.081 & 0.008 & 0.016 \\
         & FLRW$_1$  & 0.014 & 2.827 & 0.080 & 0.141 & 0.024 & 0.032 \\
         & FLRW$_2$  & 0.019 & 7.844 & 0.170 & 0.251 & 0.040 & 0.051 \\
         & FLRW$_3$  & 0.035 & 14.96 & 0.466 & 0.545 & 0.056 & 0.070 \\
         & FLRW$_4$  & 0.058 & 31.73 & 1.518 & 1.207 & 0.072 & 0.092 \\
         & FLRW$_5$  & 0.147 & 64.57 & 4.629 & 2.991 & 0.089 & 0.110 \\
         & FLRW$_6$  & 0.461 & 107.8 & 13.58 & 6.476 & 0.104 & 0.128 \\
         & FLRW$_7$  & 1.127 & 170.9 & 37.72 & 12.07 & 0.124 & 0.150 \\
         & FLRW$_8$  & 2.491 & 346.0 & 95.03 & 24.48 & 0.141 & 0.172 \\
         & FLRW$_9$  & 4.809 & 646.8 & 278.4 & 56.62 & 0.161 & 0.192 \\
         & SP        & 0.120 & 27.82 & 0.907 & 3.121 & 1.332 & 1.459 \\
         & GL        & 0.011 & 3.512 & 0.205 & 0.354 & 0.186 & 0.310 \\
			\bottomrule
		\end{tabular}
		\label{tab:times_discrete}
	\end{center}
\end{table*}

\subsection{Fixed length walk kernels for graphs with discrete labels (Q2a)}\label{sec:evaluation:flrw}
Our comparison in Section~\ref{sec:evaluation:attr} showed that computation by explicit 
feature maps becomes prohibitive when vertex and edge kernels with feature 
vectors having multiple non-zero components are multiplied. This is even observed
for the shortest-path kernel, which applies a walk kernel of fixed length one.
Therefore, we study the implicit and explicit computation schemes of the 
fixed length walk kernel on graphs with discrete labels, which are compared by
the Dirac kernel, cf.\@ Equation~\eqref{eq:two-walk-kernel-dirac}.
Since both computation schemes produce the same kernel matrices, our main focus 
in this section is on running times.

The discussion of running times for walk kernels in 
Sections~\ref{sec:flrw:kernel-computation:implicit} 
and~\ref{sec:flrw:kernel-computation:explicit}
suggested that
\begin{enumerate}[label=(\roman*)]
 \item implicit computation benefits from sparse vertex and edge kernels,
 \item explicit computation is promising for graphs with a uniform label 
       structure, which exhibit few different features, and then scales to 
       large data sets. 
\end{enumerate}
We experimentally analyze this trade-off between label diversity and running
time for synthetic and real-world data sets ignoring any attributes, if present.
Finally, we use our walk kernels to compare graphs after applying different levels 
of label refinement using the Weisfeiler-Lehman method.
We used the same experimental method as reported in Section~\ref{sec:evaluation:overall}.

\subsubsection{Synthetic data sets}

In order to systematically vary the label diversity we generated synthetic graphs
by the following procedure: The number of vertices was determined by a Poisson 
distribution with mean $20$. Edges were inserted between a pair of vertices with
probability $0.1$. The label diversity depends on the parameter $p_V$. Edges were
uniformly labeled; a vertex obtained the label $0$ with probability $1-p_V$. 
Otherwise the labels $1$ or $2$ were assigned with equal probability. In addition,
we vary the data set size $d$ between $100$ and $300$ adding $20$ randomly 
generated graphs in each step. 

\begin{figure}
   \centering
   \subfigure[Implicit computation]{\label{fig:rw_implicit_explicit:implicit}
      \includegraphics[scale=.87]{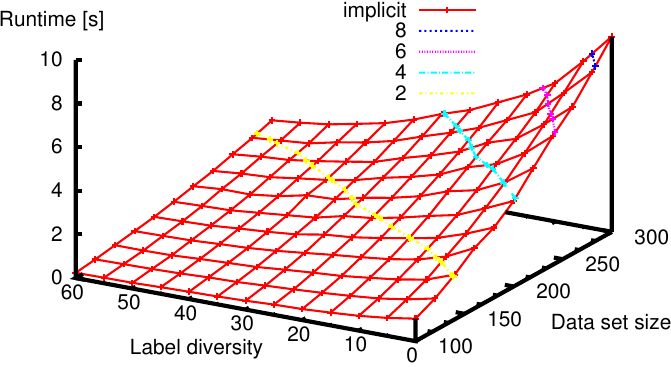}
   } \hfill
   \subfigure[Explicit computation]{\label{fig:rw_implicit_explicit:explicit}
      \includegraphics[scale=.87]{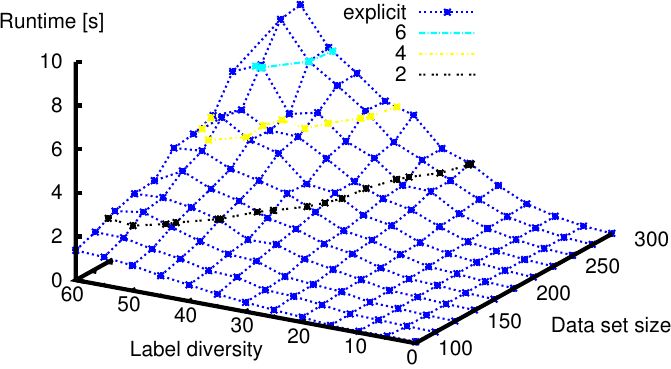}
   } \hfill
   \subfigure[Implicit and explicit computation]{\label{fig:rw_implicit_explicit:both}
      \includegraphics[scale=.87]{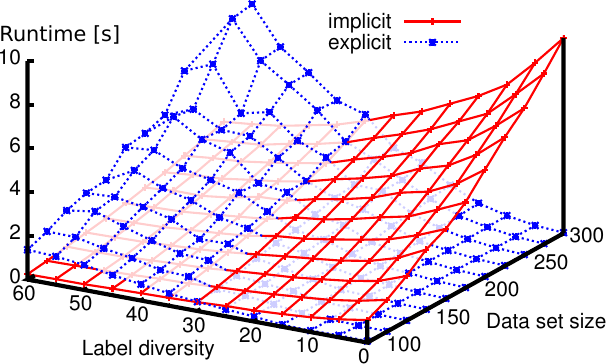}
   }
   \caption{Running time to generate the kernel matrix by implicit and explicit 
     computation of walk kernels with fixed length $7$ for synthetic data sets
     with varying label diversity.
     Figures~\subref{fig:rw_implicit_explicit:implicit} and~\subref{fig:rw_implicit_explicit:explicit}
     show contour lines obtained by linear interpolation, \subref{fig:rw_implicit_explicit:both} 
     shows the two approaches in direct comparison.
   } 
   \label{fig:rw_implicit_explicit}
\end{figure}

The results are depicted in Figure~\ref{fig:rw_implicit_explicit}, where a label
diversity of $50$ means that $p_V=0.5$.
Figure~\ref{fig:rw_implicit_explicit:implicit}
shows that the running time for implicit computation increases with the data set 
size and decreases with the label diversity. 
This observation is in accordance with our hypotheses. When the label diversity
increases, there are less compatible pairs of vertices and the weighted direct 
product graph becomes smaller. Consequently, its computation and the counting of 
weighted walks require less running time.
For explicit computation we observe a different trend: While the running time
increases with the size of the data set, the approach is extremely
efficient for graphs with uniform labels ($p_V=0$) and becomes slower when the
label diversity increases, cf.\@ Figure~\ref{fig:rw_implicit_explicit:explicit}.
Combining both results, cf.\@ Figure~\ref{fig:rw_implicit_explicit:both}, shows that 
both approaches yield the same running time for a label diversity of 
$p_V \approx 0.3$, while for higher values of $p_V$ implicit computation is preferable
and explicit otherwise.

\subsubsection{Molecular data sets}
In the previous section we have observed how both approaches behave when the label 
diversity is varied. We use the data set \textsc{U251} of graphs derived from 
small molecules to analyze the running time on a real-world data set with a 
predetermined label diversity.
Vertex labels correspond to the atom types and edge labels represent single, double,
triple and aromatic bonds, respectively.
This time we vary the walk length and the data set size by starting with a random
subset and adding additional graphs that were selected randomly from the remaining
graphs of the data set.

\begin{figure}
   \centering
   \subfigure[Implicit computation]{\label{fig:rw_implicit_explicit:length:implicit}
      \includegraphics[scale=.87]{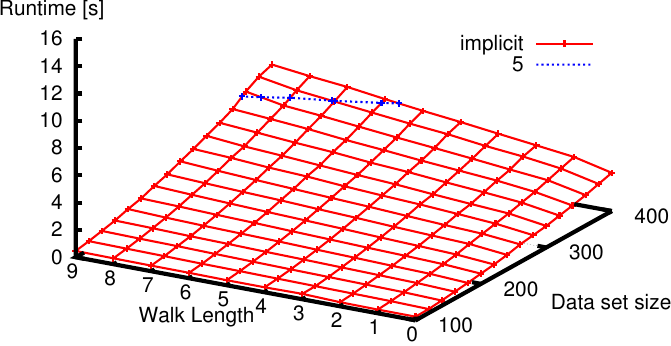}
   } \hfill
   \subfigure[Explicit computation]{\label{fig:rw_implicit_explicit:length:explicit}
      \includegraphics[scale=.87]{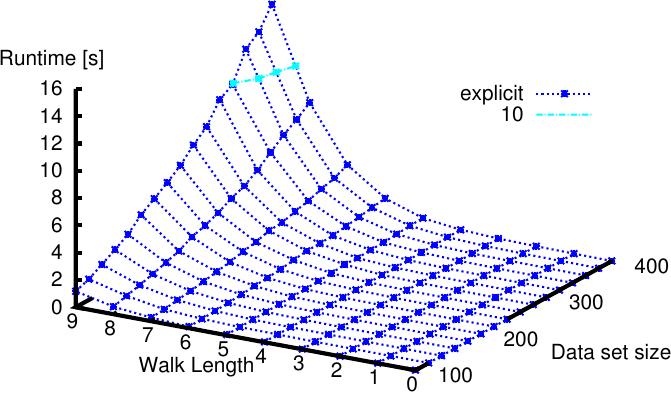}
   } \hfill
   \subfigure[Implicit and explicit computation]{\label{fig:rw_implicit_explicit:length:both}
      \includegraphics[scale=.87]{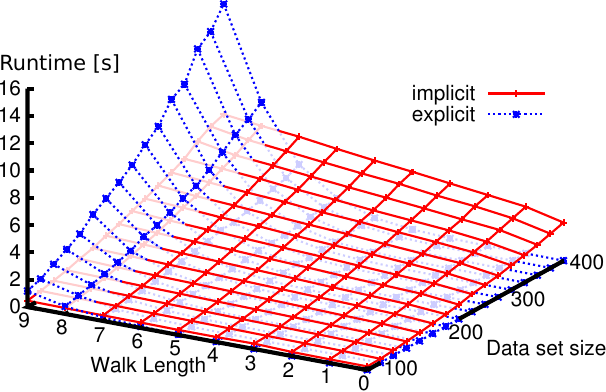}
   }
   \caption{Running time to generate the kernel matrix by implicit and explicit 
     computation of walk kernels with varying length for the molecular data set \textsc{U251}.
     Figures~\subref{fig:rw_implicit_explicit:length:implicit} and~\subref{fig:rw_implicit_explicit:length:explicit}
     show contour lines obtained by linear interpolation, \subref{fig:rw_implicit_explicit:length:both} 
     shows the two approaches in direct comparison.
   }
   \label{fig:rw_implicit_explicit:length}
\end{figure}

Figure~\ref{fig:rw_implicit_explicit:length:implicit} shows that the running time
of the implicit computation scheme heavily depends on the size of the data set.
The increase with the walk length is less considerable. This can be explained by
the time $\mathsf{T}_{\WDPG}$ required to compute the product graph, which is always needed 
independent of the walk length.
For short walks explicit computation is very efficient, even for larger data 
sets, cf.\@ Figure~\ref{fig:rw_implicit_explicit:length:explicit}.
However, when a certain walk length is reached the running time increases 
drastically. This can be explained by the growing number of different label 
sequences. Notably for walks of length $8$ and $9$ the running
time also largely increases with the data set size. This indicates that the time
$\mathsf{T}_{\text{dot}}$ has a considerable influence on the running time. In the following 
section we analyze the running time of the different procedures for the two
algorithms in more detail.
Figure~\ref{fig:rw_implicit_explicit:length:both} shows that for walk length up
to $7$ explicit computation beats implicit computation on the molecular data
set.

\subsubsection{\textsc{Enzymes} and \textsc{Mutag}}

We have shown that up to a certain walk length explicit computation is more
efficient than implicit computation. We want to clarify the relation between
the walk length and the prediction accuracy in a classification task. 
In addition, we analyze the ratio between the time $\mathsf{T}_\phi$ for computing the 
explicit mapping and $\mathsf{T}_{\text{dot}}$ for taking dot products.
For the implicit computation scheme we want to clarify the running time of 
$\mathsf{T}_{\WDPG}$ and the time required for counting weighted walks. 
We apply both algorithms to two widely used data sets, \textsc{Mutag} and 
\textsc{Enzymes}, and vary the walk length, see Table~\ref{tab:datasets} for 
details on these data sets.

\begin{figure}[tb]
\vspace{-0.12cm}
  \centering
  \subfigure[\textsc{Mutag}]{\label{fig:runtime_accuracy_detail:mutag}
    \includegraphics[scale=.71]{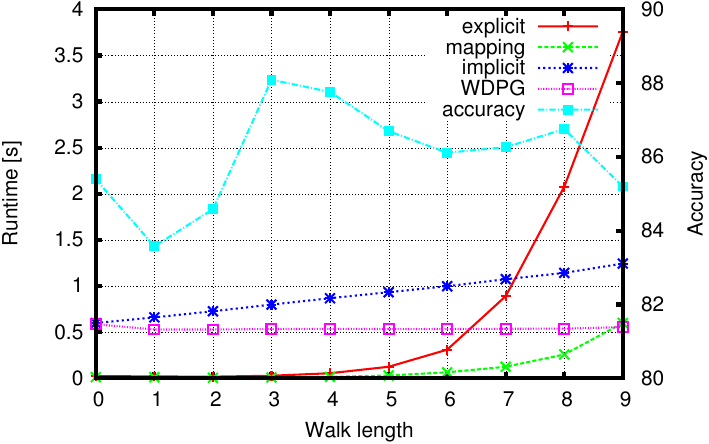}
  } \hfill
  \subfigure[\textsc{Enzymes}]{\label{fig:runtime_accuracy_detail:enzyme}
    \includegraphics[scale=.71]{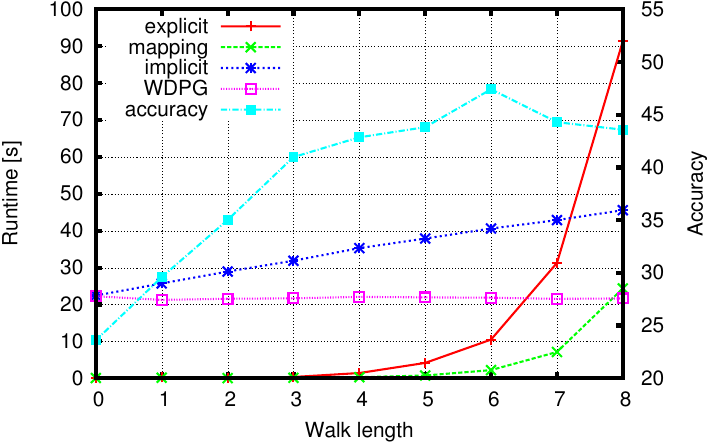}
  }
  \caption{Running time to generate the kernel matrix by explicit and implicit 
    computation on the \textsc{Enzymes} and \textsc{Mutag} data sets depending 
    on the walk length.
    Both approaches compute the same kernel matrix which leads to the accuracy 
    plotted on the right y-axis.}
   \label{fig:runtime_accuracy_detail}
\end{figure}

Figure~\ref{fig:runtime_accuracy_detail} shows the running time of both algorithms
depending on the walk length and gives the time for product graph computation
and explicit mapping, respectively. In addition, the prediction accuracy is 
presented.
For both data sets we observe that up to a walk length of $7$ explicit mapping is
more efficient. Notably a peak of the accuracy is reached for walk length smaller
than $7$ in both cases. For the \textsc{Mutag} data set walks of length $3$ provide the 
best results and walks of length $6$ for the \textsc{Enzymes} data set, i.e., in both cases
explicit mapping should be preferred when computing a walk kernel of fixed length.
The running time of the product graph computation is constant and does not 
depend on the walk length. For explicit mapping the time required to compute
the dot product becomes dominating when the walk length is increased. This can 
be explained by the fact that the generation of the kernel matrix involves a 
quadratic number of dot product computations.
Note that the given times include a quadratic number of product graph computations
while the times for generating the feature vectors include only a linear number of
operations.

As a side note, we also compared the accuracy of the fixed length walk kernels
to the accuracy reached by the geometric random walk kernel (GRW)
according to \citet{Gaertner2003}, which considers arbitrary 
walk lengths. 
The parameter $\gamma$ of the geometric random walk kernel was selected by 
cross-validation from $\{10^{-5}, 10^{-4},\dots,10^{-2}\}$.
We observed that the accuracy of the fixed length walk kernel is competitive on 
the \textsc{Mutag} data set 
(GRW $87.3$), 
and considerably better on the \textsc{Enzymes} data set 
(GRW $31.6$), 
cf.\@ Figure~\ref{fig:runtime_accuracy_detail}. 
This is remarkable, since the fixed length walk kernel yields best results with 
walk length $6$, for which it is efficiently computed by explicit mapping. 
However, this is not possible for the geometric random walk kernel.
For a more detailed discussion and comparison between fixed length walk kernels 
and the geometric random walk kernel we refer the reader to~\citep{Sug+2015}, 
which appeared after the conference publication~\citep{Kri+2014}.

\subsubsection{Weisfeiler-Lehman label refinement}
Walk kernels have been successfully combined with label refinement techniques~\citep{Mah'e2004}.
We employ the Weisfeiler-Lehman label refinement (WL) as described in
Section~\ref{sec:wv}.
To further analyze the sensitivity w.r.t.\ label diversity, we again use  the 
\textsc{Enzymes} data set, which consists of graphs with three vertex and two edge labels 
initially. We apply our algorithms for explicit and implicit computation after $0$ to $3$ iterations of WL, see 
Figure~\ref{fig:rw_implicit_explicit:length:wl}.

\begin{figure}
   \centering
      \includegraphics[scale=.87]{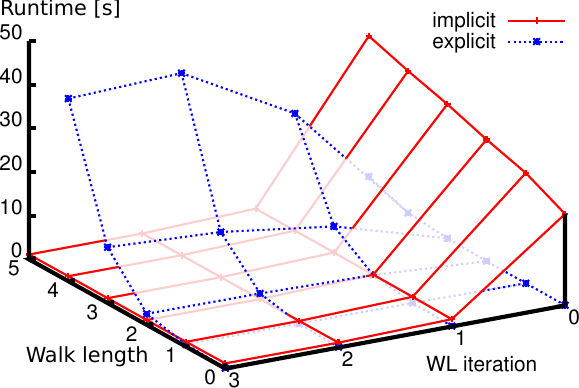}
   \caption{Running time to generate the kernel matrix by implicit and explicit 
     computation of walk kernels with varying walk length and iterations of
     Weisfeiler-Lehman refinement on the \textsc{Enzymes} data set.}
   \label{fig:rw_implicit_explicit:length:wl}
\end{figure}

If no refinement is applied, the explicit mapping approach beats the product
graph based algorithm for the used walk lengths. However, as soon as a single
iteration of label refinement is performed, the product graph based algorithm
becomes competitive for walk length $0$ and $1$ and outperforms the explicit
mapping approach for higher walk lengths. The running times do not change
substantially for more iterations of refinement. This indicates that a single
iteration of Weisfeiler-Lehman refinement results in a high label diversity
that does not increase considerably for more iterations on the \textsc{Enzymes} data set.
When using our walk-based kernel as base kernel of a Weisfeiler-Lehman graph 
kernel~\citep{Shervashidze2011}, our observation suggests to start with explicit 
computation and switch to the implicit computation scheme after few iterations 
of refinement.

\subsection{Subgraph kernels for graphs with discrete labels (Q2b)}\label{sec:evaluation:sm}
In this section we experimentally compare the running time of the subgraph matching
and the subgraph (or graphlet) kernel as discussed in Section~\ref{sec:subgraph:exp-impl}.
The explicit computation scheme, which is possible for graphs with discrete 
labels compared by the Dirac kernel, is expected to be favorable.

\subsubsection{Method}
We have reimplemented a variation of the graphlet kernel taking connected induced 
subgraphs with three vertices and discrete vertex and edge labels into account.
The only possible features are triangles and paths of length two. Graph 
canonization is realized by selecting the lexicographically smallest string 
obtained by traversing the graph and concatenating the observed labels. 
Our implementation is similar to 
the approach used by~\citet{Shervashidze2011} as extension of the original 
graphlet kernel~\citep{Shervashidze2009} to the domain of labeled graphs. We 
refer to this method as graphlet kernel in the following.
We compared the graphlet kernel to the connected subgraph matching kernel 
taking only connected subgraphs on three vertices into account.
In order not to penalize the running time of the connected subgraph matching 
kernel by additional automorphism computations, the weight function does not 
consider the number of automorphisms \citep[Theorem 2]{Kriege2012} and, 
consequently, not the same kernel values are computed.

The experiments were conducted using Sun Java JDK v1.6.0 on an Intel Xeon E5430 
machine at 2.66GHz with 8GB of RAM using a single processor only.
The reported running times are average values over $5$ runs.

\subsubsection{Results and discussion}
For real-world instances we observed that explicit computation outperforms
implicit computation by several orders of magnitude, cf.\@ Table~\ref{tab:times_discrete}.
This in accordance with our theoretical analysis. However, the practical 
considerations suggest that explicit and implicit computation behave 
complementary and subgraph matching kernels become competitive if a sufficient 
small and sparse weighted product graph is generated, which occurs for graphs
with increasing label diversity as for the walk-based kernels. Hence, we 
randomly generated graphs with the following procedure:
The number of vertices was determined by a Poisson distribution with mean $60$. 
Edges were inserted between a pair of vertices with probability $0.5$. Labels
for vertices and edges were assigned with equal probability, whereas the size
of the label alphabet $\mathcal{L}=\mathcal{L}_V=\mathcal{L}_E$ is varied from 
$1$, i.e., uniform labels, to $65$. Note that the graphs obtained by this 
procedure have different characteristics than those used to show the 
computational phase transition for walk-based kernels.
We vary the data set size $d$ between $100$ and $300$ adding $50$ randomly 
generated graphs in each step and analyze the running time to compute the 
$d \times d$ kernel matrix. For the subgraph matching kernel we used the Dirac 
kernel as vertex and edge kernel.

\begin{figure}
   \centering
   \subfigure[Subgraph matching kernel (implicit)]{\label{fig:sm_implicit_explicit:implicit}
      \includegraphics[scale=.87]{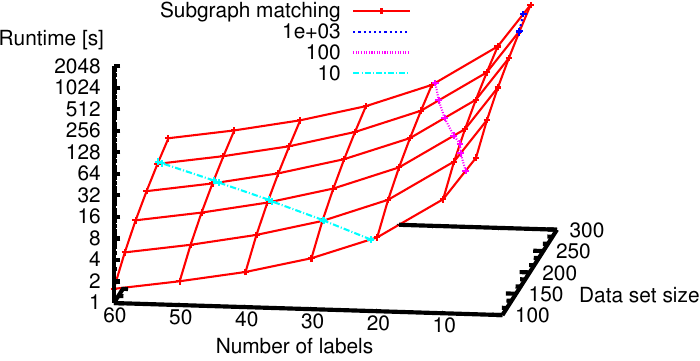}
   } \hfill
   \subfigure[Graphlets with three vertices (explicit)]{\label{fig:sm_implicit_explicit:explicit}
      \includegraphics[scale=.87]{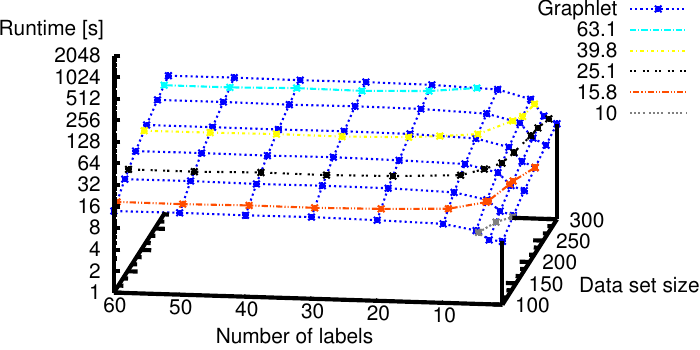}
   } \hfill
   \subfigure[Implicit and explicit computation]{\label{fig:sm_implicit_explicit:both}
      \includegraphics[scale=.87]{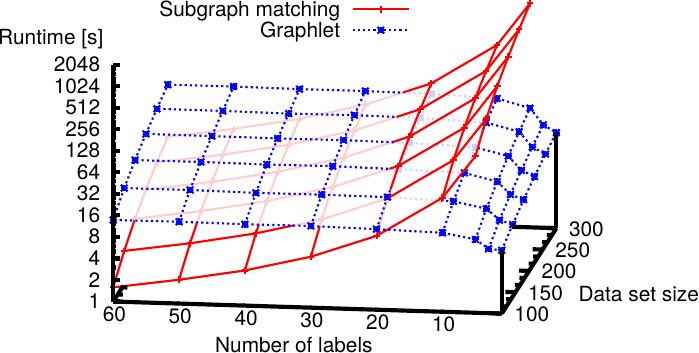}
   }
   \caption{Running time to generate the kernel matrix by implicit and explicit 
     computation for synthetic data sets with varying size of the label alphabet.
     Figures~\subref{fig:sm_implicit_explicit:implicit} and~\subref{fig:sm_implicit_explicit:explicit}
     show contour lines obtained by linear interpolation, \subref{fig:sm_implicit_explicit:both} 
     shows the two approaches in direct comparison.
   }
   \label{fig:sm_implicit_explicit}
\end{figure}

Figure~\ref{fig:sm_implicit_explicit} shows a computational phase transition:
For the synthetic data set the subgraph matching kernel is more efficient than
the graphlet kernel for instances with 20-30 different labels and its running 
time increases drastically when the number of labels decreases.
The graphlet kernel in contrast is more efficient for graphs with uniform or 
few labels. For more than 10 different labels, there is only a moderate increase 
in running time. This can be explained by the
fact that the number of features contained in the graphs does not increase 
considerably as soon as a certain number of different labels is reached. The
enumeration of triangles dominates the running time for this relatively dense
synthetic data set. The running time behavior of the subgraph matching kernel is 
as expected and is directly related to the size and number of edges in the 
weighted association graph.

Our synthetic data set differs from typical real-world instances, since we
generated dense graphs with many different labels, which are assigned uniformly
at random. For real-world data sets the graphlet kernel consistently outperforms
the subgraph matching kernel by orders of magnitude.
It would be interesting to further investigate where this computational phase 
transition occurs for larger subgraphs and to analyze if the implicit 
computation scheme then becomes competitive for instances of practical relevance. 
This requires the implementation of non-trivial graph canonization algorithms 
and remains future work. The results we obtained clearly suggest to prefer the 
explicit computation schemes when no flexible scoring by vertex and edge 
kernels is required.

\section{Conclusion}
The breadth of problems requiring to deal with graph data is growing rapidly and
graph kernels have become an efficient and widely used method for measuring 
similarity between graphs. 
Highly scalable graph kernels have recently been proposed for graphs with 
thousands and millions of vertices based on explicit graph feature maps. 
Implicit computation schemes are used for kernels with a large number of possible
features such as walks and when graphs are annotated by continuous attributes.

To set the stage for the experimental comparison, we actually made several
contributions to the theory and algorithmics of graph kernels. 
We presented a unified view on implicit and explicit graph features. More precisely, 
we derived explicit feature maps from the implicit feature space of convolution 
kernels and analyzed the circumstances rendering this approach feasible in practice. 
Using these results, we developed explicit computation schemes for random walk 
kernels \citep{Gaertner2003,Vishwanathan2010}, subgraph matching kernels 
\citep{Kriege2012}, and shortest-path kernels \citep{Borgwardt2005}. 
Moreover, we introduced weighted vertex kernels and derived explicit feature 
maps. As a result of this we obtained approximate feature maps for state-of-the-art 
kernels for graphs with continuous attributes such as the GraphHopper kernel~\citep{Feragen2013}.
For fixed length walk kernels we have developed implicit and explicit computation
schemes and analyzed their running time. Our theoretical results have been 
confirmed experimentally by observing a computational phase transition with 
respect to label diversity and walk lengths.

We have shown that kernels composed by multiplication of non-trivial base 
kernels may lead to a rapid growth of the number of non-zero components in 
the feature vectors, which renders explicit computation infeasible. One approach 
to alleviate this in future work is to introduce sampling or hashing to obtain compact feature 
representations in such cases, e.g., following the work by~\citet{Shi2009}.

\bibliographystyle{unsrtnat}
\bibliography{feature_maps}

\end{document}